\documentclass[11pt]{article}

\usepackage[final]{acl}

\usepackage{times}
\usepackage{latexsym}

\usepackage[T1]{fontenc}

\usepackage[utf8]{inputenc}

\usepackage{microtype}

\usepackage{inconsolata}

\usepackage{graphicx}

\usepackage{amsmath}
\usepackage[capitalise]{cleveref}
\usepackage{amssymb}
\usepackage{bbm}
\usepackage{booktabs}
\usepackage{tcolorbox}
\usepackage{xspace}
\usepackage{subcaption}
\usepackage{multirow}
\usepackage{pgfplots}
\pgfplotsset{compat=1.18}
\usepackage{amsthm}
\usepackage{thmtools}
\usepackage{thm-restate}
\usepackage{algorithm}
\usepackage{algpseudocode}

\definecolor{quorum}{HTML}{DC267F}
\definecolor{random}{HTML}{785EF0}
\definecolor{pac}{HTML}{648FFF}
\definecolor{llm}{HTML}{FE6100}
\definecolor{llm_human}{HTML}{FFB000}

\newtheorem{theorem}{Theorem}[section]
\newtheorem{proposition}[theorem]{Proposition}

\newcommand{\quorum}{QUORUM\xspace}
\newcommand{\approach}{\textbf{\quorum}\xspace}
\everypar\expandafter{\the\everypar\looseness=-1}

\title{\approach: QUality-Optimized Routing Using Multiple annotators}

\author{
 \textbf{Antonio Purificato\textsuperscript{1,2}},
 \textbf{Maria Sofia Bucarelli\textsuperscript{3,$\star$}},
 \textbf{Andrea Bacciu\textsuperscript{1}},\\
 \textbf{Amin Mantrach\textsuperscript{1}},
 \textbf{Fabrizio Silvestri\textsuperscript{2}}
\\
\\
 \textsuperscript{1} Amazon,
 \textsuperscript{2} Sapienza University of Rome,
 \textsuperscript{3}Université Côte d’Azur, CNRS, Inria, I3S
\\
 \small{
   \textbf{Correspondence:} \href{mailto:email@domain}{purifian@amazon.com}
 }
}

\begin{document}
\maketitle
\begingroup
\renewcommand{\thefootnote}{}
\footnotetext{$^{\star}$ Work done while at Sapienza.}
\endgroup
\begin{abstract}
Data annotation remains a central bottleneck in natural language processing, requiring human effort to obtain high-quality labels at scale. While Large Language Models (LLMs) offer a fast and cost-effective alternative, their reliability is highly instance-dependent: they perform well on simple inputs but often fail on examples requiring nuanced reasoning or contextual understanding.
In this work, we address this challenge with \approach (\textbf{QU}ality-\textbf{O}ptimized \textbf{R}outing \textbf{U}sing \textbf{M}ultiple annotators), a budget-aware routing framework that dynamically assigns each instance to human or LLM annotators under a fixed annotation budget.
Unlike prior approaches relying on model confidence or uncertainty estimates, \quorum leverages feature-based signals to estimate instance difficulty and supports multiple annotations per instance, combining them through agreement-based rewards to improve reliability.
We evaluate \quorum across diverse closed- and open-ended annotation tasks in English and multilingual settings, and \quorum improves annotation quality by up to 34.4\% while reducing costs by 8.8\% over competing methods. Code can be found at \url{https://github.com/amazon-science/QUORUM}.
\end{abstract}

\section{Introduction}

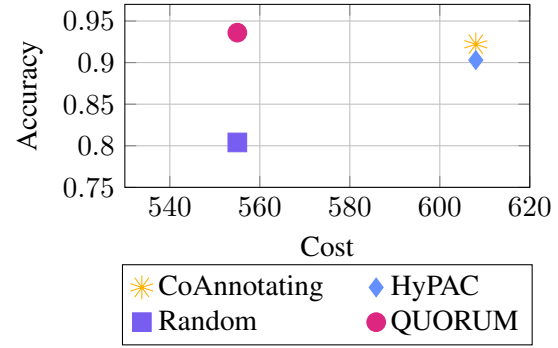
\begin{figure}[ht]
\centering
\begin{tikzpicture}

\begin{axis}[
    width=0.9\columnwidth,
    height=4cm,
    grid=both,
    xlabel={Cost},
    ylabel={Accuracy},
    xmin=530, xmax=620,
    ymin=0.75, ymax=0.97,
    legend style={
        at={(0.5,-0.42)},
        anchor=north,
        legend columns=2,
        /tikz/every even column/.append style={column sep=0.5cm}
    },
    legend cell align=left,
]


\addplot[
    only marks,
    mark=10-pointed star,
    mark size=4.5pt,
    color=llm_human
] coordinates {
    (608,0.922)
};
\addlegendentry{CoAnnotating}

\addplot[
    only marks,
    mark=diamond*,
    mark size=3.5pt,
    color=pac
] coordinates {
    (608,0.903)
};
\addlegendentry{HyPAC}

\addplot[
    only marks,
    mark=square*,
    mark size=3.5pt,
    color=random
] coordinates {
    (555,0.804)
};
\addlegendentry{Random}

\addplot[
    only marks,
    mark=*,
    mark size=3.5pt,
    color=quorum
] coordinates {
    (555,0.936)
};
\addlegendentry{\quorum}

\end{axis}

\end{tikzpicture}
\caption{Pareto frontier of cost versus performance on AG's News dataset with 5,320 human annotations. The standalone LLM achieves 0.834 average accuracy. Costs denote cumulative annotation expenses.}
\label{fig:costs}
\end{figure}

Human annotation is costly and difficult to scale across domains and languages\footnote{\url{sourcebae.com/blog/what-is-data-annotation}}. LLMs offer a faster and cheaper alternative, although annotation quality can vary substantially~\citep{shimizu2025exploring,tan-etal-2024-large}. In practice, some instances can be reliably annotated by LLMs, whereas others require nuanced reasoning, domain expertise, or cultural understanding~\citep{hendrycksmeasuring}. The challenge is therefore no longer whether humans or LLMs should annotate data, but \textit{how to allocate annotation effort between them under limited budgets}.
A natural solution is to dynamically route instances between annotators. Existing approaches largely rely on uncertainty estimation, using model confidence to decide when human supervision is required~\citep{li2023coannotating,gligoric2025can}. While effective, these methods often depend on calibrated confidence scores, additional inference steps, or model-specific uncertainty signals, increasing computational overhead and limiting applicability across annotators~\citep{cecere2025monte}.

In this work, we propose \approach (\textbf{QU}ality-\textbf{O}ptimized \textbf{R}outing \textbf{U}sing \textbf{M}ultiple annotators), a budget-aware routing framework that dynamically assigns instances to human or LLM annotators to jointly optimize annotation quality and cost. Unlike prior approaches, \quorum does not rely on uncertainty estimation. Instead, it learns routing policies from contextual and budget signals, deciding not only \emph{who} should annotate an instance, but also \emph{whether additional annotations are worth their cost}. This naturally enables multi-annotator strategies, where multiple LLMs or humans can be combined when agreement is insufficient. In fact, the name \quorum is inspired by collective decision-making, where a \textit{quorum} represents the minimum agreement required to reach a reliable outcome \footnote{\url{dictionary.cambridge.org/dictionary/english/quorum}}. Similarly, \quorum aggregates annotations from humans and LLMs, dynamically requesting additional supervision only when needed.

\Cref{fig:costs} previews this trade-off on the AG's News dataset: \quorum achieves the highest annotation quality while matching the cost of the cheapest competing methods, placing it on the Pareto frontier. Our main contributions are:
\begin{itemize}
    \item We propose \quorum, a budget-aware routing framework for allocating annotations between human and LLM annotators, supporting multi-annotator aggregation when needed.
    \item We provide a theoretically grounded formulation of budget-aware routing, with guarantees on posterior concentration and asymptotic convergence toward oracle routing.
    \item \quorum achieves up to 34.4\% relative improvement in annotation quality while reducing cost by up to 8.8\% across English and multilingual open- and closed-ended tasks.
\end{itemize}

\section{Related Work}
\label{sec:related_work}

Recent work on human–LLM collaboration for annotation has increasingly focused on balancing annotation quality and cost through adaptive routing strategies~\citep{tan-etal-2024-large}.
A common starting point is the use of uncertainty or confidence signals to decide whether a sample should be handled by a model or escalated to a human annotator~\citep{purificato2026select}.

\paragraph{Uncertainty-based routing.}
Early approaches, such as those proposed by~\citet{wang2021want} and extended by~\citet{wang2024human}, rely on model-provided confidence scores (\textit{e.g.}, logits from LLM APIs) to determine whether an LLM prediction is sufficiently reliable. Samples with low confidence are routed to human annotators. While intuitive, this strategy has two main limitations: (i) it depends on the availability of calibrated confidence scores, which are not consistently provided across models, and (ii) it requires an additional model inference step before routing decisions can be made, introducing non-negligible computational overhead.
To overcome these limitations, subsequent work has proposed alternative uncertainty estimation mechanisms.
~\citet{li2023coannotating} introduce a human–LLM co-annotation framework that models uncertainty via self-evaluated confidence scores and entropy over multiple model outputs, formulating the problem as a multi-objective optimization task over annotation quality and cost. Similarly,~\citet{gligoric2025can} propose Confidence Driven Inference (CDI), where models explicitly verbalize confidence estimates that are then used to decide when human supervision is required.
Finally,~\citet{zeng2026hypac} propose HyPAC, a  framework with PAC-style guarantees that partitions the input space using uncertainty and applies importance sampling with upper confidence bounds to optimize annotation allocation under budget constraints.
\paragraph{Alternative routing strategies.}
Beyond confidence-based methods, other approaches incorporate external knowledge or learned error signals to improve routing decisions.
~\citet{huang2024araida} propose ARAIDA, which augments model predictions with retrieval from nearest labeled neighbors and adaptively weights model and retrieval signals based on estimated error. In a related direction,~\citet{huang2024selective} introduce SANT, which integrates active learning with error-aware triage mechanisms to prioritize samples for human annotation under limited budgets.
More recent work explores learning-based and theoretically grounded routing policies.
~\citet{miranda2025hybrid} learn a performance prediction model that selects the optimal combination of human and LLM annotations based on a preference dataset, effectively learning routing decisions from data.

Overall, \quorum avoids dependence on model-derived uncertainty and additional inference steps, relying on feature-based difficulty signals to enable a more general and cost-efficient routing strategy under budget constraints.

\section{Method}

\begin{figure*}
    \centering
    \includegraphics[width=\linewidth]{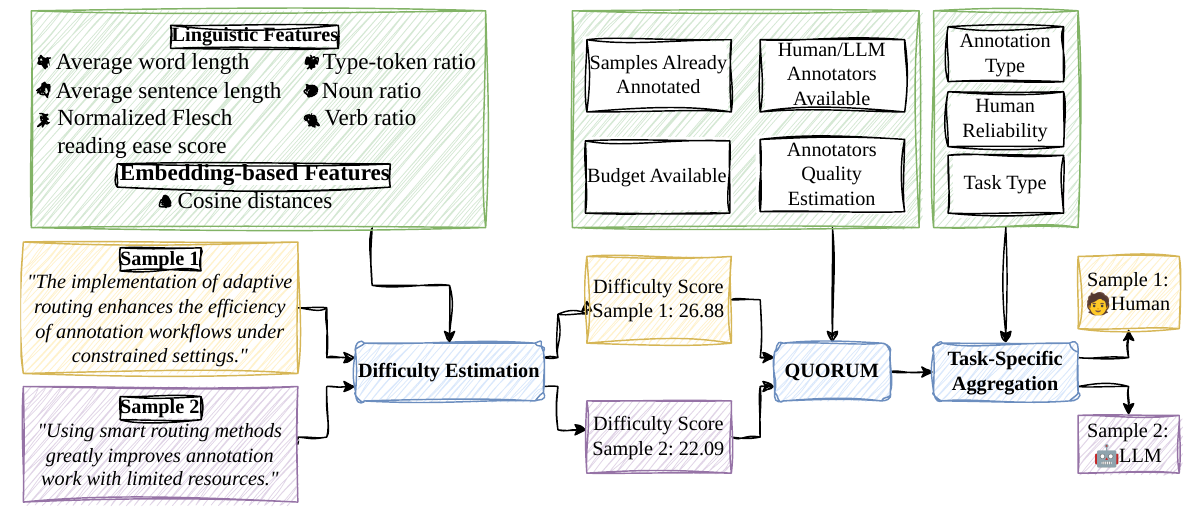}
    \caption{\quorum, a budget-aware annotation router estimating instance difficulty using linguistic and semantic features, assigning samples to human or LLM annotators to maximize annotation quality while minimizing cost.}
    \label{fig:graphical_abstract}
\end{figure*}

We consider a dataset $\mathcal{D}$ of $T$ samples:
\begin{equation*}
    \mathcal{D} = \{d_t\}_{t=1}^T \quad d_t \in \mathbb{R}^d
\end{equation*}
At each timestep $t$, the router selects annotators from an arm set, where an arm denotes an available action whose selection yields a reward and an associated cost~\citep{lattimore2020bandit}:
\begin{equation*}
   \mathcal{A} = \{1, \ldots, K, H\}
\end{equation*}
Where $K$ denotes the number of LLM annotators and $H$ denotes the human annotator.
Each arm $a \in \mathcal{A}$ is associated with a cost $c_a > 0$.
The objective is to maximize annotation quality under a budget constraint $B$. \Cref{fig:graphical_abstract} provides an overview of the proposed \quorum framework, including difficulty estimation, adaptive routing, budget-aware adaptation, and task-specific aggregation.

\subsection{Feature Definition}
\label{sec:features}

Our approach relies on feature-based signals to estimate the complexity of each input instance. These features serve as a proxy for instance difficulty and replace the use of model-derived uncertainty employed in prior work~\citep{zeng2026hypac}.

For each sample $d_t$, we construct a feature representation $\phi(d_t) \in \mathbb{R}^m$ based on two components.

\paragraph{Linguistic features}
We extract a set of scalar features designed to capture lexical and syntactic complexity~\citep{ziller2026greenserv, yang2019predicting}. These include: (i) average word length, (ii) type-token ratio, (iii) average sentence length, (iv) normalized Flesch reading ease score~\citep{flesch1948new}, (v) noun ratio, and (vi) verb ratio. Together, these features characterize properties such as vocabulary sophistication, structural complexity, and information density, which are known to correlate with annotation difficulty.

\paragraph{Embedding-based features}
To capture semantic properties, each instance is encoded into a dense vector representation using a pretrained text multilingual encoder\footnote{\url{huggingface.co/NovaSearch/stella_en_1.5B_v5}}. We then compute the cosine distances ($1 -$ cosine similarities) between $d_t$ and its $k$ nearest neighbors in the embedding space.
The resulting distance vector reflects the degree of semantic isolation of the instance, with higher distances indicating atypical or out-of-distribution samples potentially more difficult to annotate~\citep{hassan2023d}.

The final feature vector $\phi(d_t)$ is obtained by concatenating the linguistic and embedding-based features and applying normalization. This representation is used as input to the routing policy. Additional feature details are reported in~\cref{app:feature_description}.

\paragraph{Context Construction}
For each sample, the router first builds a contextual representation:
\begin{equation*}
x_t=[\phi(d_t), b_t, p_t, m_t]
\end{equation*}
Where $\phi(d_t)$ denotes normalized input features, $b_t$ is the remaining budget fraction, $p_t=t/T$ represents progress through the dataset, and $m_t$ is a binary mask indicating which LLMs have already been queried for that sample.
Consequently, each routing decision depends not only on the sample itself, but also on previously consumed resources and earlier annotations. This contextual vector becomes the input to all subsequent estimation stages.

\subsection{Human-LLM Routing with Calibration}
\label{routing_human_llm}

\paragraph{Feature-Space Uncertainty}
Before routing begins, the model allocates approximately 20\% of the affordable human budget to a calibration stage. This step mitigates uncertainty in early routing decisions by establishing initial estimates of LLM reliability. The objective is to estimate the reliability of each LLM before online decisions are made.
Calibration samples are selected according to feature-space uncertainty:
\begin{equation*}
u_t=\sqrt{(\phi(d_t)-\mu)^T\Sigma^{-1}(\phi(d_t)-\mu)}
\end{equation*}

where $\mu$ and $\Sigma$ are statistics over contextual features. Samples with highest uncertainty receive human labels. Human supervision is then used to estimate correctness of each LLM prediction:
\begin{equation*}
r_{a,t}=\mathbbm{1}(\hat y^{(a)}_t=y_t)
\end{equation*}

The resulting calibration rewards initialize annotator-specific posteriors.

\paragraph{Annotator Quality Estimation}

Each arm maintains Gaussian posterior parameters:
\begin{equation*}
\theta_a\sim\mathcal N(\mu_a,\Sigma_a)
\end{equation*}

Updated using Bayesian linear regression:

\begin{equation*}
\Sigma'_a= \left( \Sigma_a^{-1} + \frac{xx^T}{\sigma^2} \right)^{-1}
\end{equation*}

\begin{equation*}
\mu'_a= \Sigma'_a \left( \Sigma_a^{-1}\mu_a + \frac{r}{\sigma^2}x\right)
\end{equation*}

Posterior means encode expected annotator quality, whereas posterior covariance captures uncertainty. Human annotators are initialized with optimistic priors (higher mean and lower variance), reflecting greater expected reliability. This follows an assumption in annotation research that human labels generally provide stronger supervision signals than automated annotators~\citep{gligoric2025can}.

The learned posterior parameters subsequently determine routing choices.

\paragraph{Routing Mechanism}
Given context $x_t$, expected quality for each available LLM is estimated as $q_a= \mu_a^Tx_t$.
To avoid premature exploitation, routing includes decaying exploration:
\begin{equation*}
\epsilon_t=\frac{\epsilon_0}{1+\gamma t}
\end{equation*}
Where $\epsilon_0$ denotes the initial exploration probability and $\gamma$ controls the rate at which exploration decreases over time.
Therefore, early decisions emphasize exploration, whereas later stages increasingly exploit calibrated annotator estimates. At each timestep,
with probability $\epsilon_t$, the router explores by randomly selecting an available LLM:
\begin{equation*}
a_t \sim \text{Uniform}(\mathcal{A}_{\text{LLM}})
\end{equation*}
Otherwise, routing exploits current quality estimates and selects the
highest-scoring annotator:
\begin{equation*}
a_t=\arg\max_a q_a
\end{equation*}

When predicted quality falls below a threshold $\tau_H$, the router escalates to human annotation:
\begin{equation*}
a_t=
\begin{cases}
H,& q_{a^*}<\tau_H\\
a^*,& \text{otherwise}
\end{cases}
\end{equation*}

\subsection{Iterative Annotation Selection}

\paragraph{Budget-Aware Adaptive Refinement}
Every routing decision updates consumed cost. Human budgets are partitioned into a working component and a reserved component $B= B_{\text{working}} + B_{\text{reserved}}$.
The reserved part prevents exhaustion of expert annotations before difficult late-stage samples appear. 


After initial routing, samples are ranked for possible additional annotations. Difficulty is estimated from posterior uncertainty:

\begin{equation*}
D_t= \frac1K \sum_a x_t^T \Sigma_a x_t
\end{equation*}

We further define the annotation disagreement among collected labels: \begin{equation*} \Delta_t = 1 - \frac{\max_y  \text{count}(y)}{|\mathcal{Y}_t|} 
\end{equation*} where $\mathcal{Y}_t$ is the set of annotations for sample $t$.
For each sample, a priority score combines difficulty, disagreement and missing human supervision:
\begin{equation*}
P_t= D_t+ \Delta_t+ \mathbbm1(\text{no human})
\end{equation*}

Samples with highest priority receive extra annotations until budgets are exhausted. Hence annotation becomes an iterative refinement process rather than a single routing decision.

\paragraph{Task-specific Aggregation}
The final stage aggregates collected annotations differently depending on task type.
For classification, agreement is computed from label frequency with human weight $w_H=10$ and LLM weight $w_{\text{LLM}}=1$.

Consensus is achieved when weighted agreement exceeds a threshold. Final outputs are determined through weighted majority voting.
For summarization, exact agreement is inappropriate. Instead, summaries are represented through embeddings and semantic agreement is measured via average cosine similarity:

\begin{equation*}
r_t = 2 \left (\frac{ \sum_{i \neq j} \cos(e_i, e_j)}{|\mathcal{Y}_t| (|\mathcal{Y}_t| - 1)}\right )- 1
\end{equation*}

Where $e_i, e_j \in \mathbb{R}^d$ denote the embeddings. Consensus is reached when semantic similarity exceeds a threshold. The final summary is selected as the annotation closest to the embedding centroid of all summaries.
Therefore, in summarization we are rewarding, due to absence of labels, semantic consistency among annotators. The output of this aggregation stage constitutes the final prediction. The pseudocode of \quorum is in~\cref{pseudocode_sec}.

\subsection{Theoretical Properties}

\begin{proposition}[Posterior concentration]
\label{prop:posterior_concentration}

Assuming bounded contextual representations \(\|x_t\|_2 \le L\) and bounded true annotator parameters \(\|\theta_a^*\|_2 \le S\), with Gaussian noise:
\begin{equation*}
r_{a,t}=x_t^\top \theta_a^* + \eta_t, \qquad \eta_t \sim \mathcal N(0,\sigma^2)
\end{equation*}
For any annotator \(a\), with probability at least
\(1-\delta\):
\begin{equation*}
\left| x_t^\top ( \mu_{a,t} - \theta_a^*) \right| \le \beta_t \sqrt{ x_t^\top \Sigma_{a,t} x_t}
\end{equation*}
With:
\begin{equation*}
\beta_t= \sigma \sqrt{d\log \left( \frac{ 1+tL^2/\lambda}{
\delta}\right)}+\sqrt{\lambda}S .
\end{equation*}
We have that:
\begin{equation*}
x_t^\top \Sigma_{a,t}x_t\rightarrow0
\end{equation*}
As \(t\rightarrow\infty\), implying that uncertainty estimates become increasingly accurate over time.
\end{proposition}
Annotator quality estimates concentrate around the true parameters as observations accumulate. Therefore, the uncertainty-driven difficulty score used by \quorum becomes progressively more reliable. A proof is provided in~\cref{app:proof_posterior}.

\begin{proposition}[Vanishing average regret]
\label{prop:regret}
Under the same assumptions of Proposition \ref{prop:posterior_concentration}, 
let:
\[
a_t^*
\in
\arg\max_{a\in\mathcal A}x_t^\top\theta_a^*
\]
be the context-dependent oracle annotator, and define:
\[
R(T)=
\sum_{t=1}^T
\left(
x_t^\top\theta_{a_t^*}^*
-
x_t^\top\theta_{\hat a_t}^*
\right)
\]
If the core routing rule uses decaying exploration
\(\epsilon_t=\epsilon_0/(1+\gamma t)\), the oracle annotator is uniformly
separated from the second-best annotator, and exploration provides sufficiently
diverse samples for each feasible annotator, then:
\[
\lim_{T\to\infty}\frac{R(T)}{T}=0
\]
Implying asymptotic convergence toward oracle routing.
\end{proposition}
Proposition~\ref{prop:regret} implies that the average loss
relative to an optimal routing policy vanishes over time.
Therefore, \quorum asymptotically approaches the performance of an oracle with knowledge of annotator quality.
Proof is provided in~\cref{app:proof_regret}.

\section{Experiments}
We evaluate \quorum across multiple tasks, budget regimes, and annotation settings. Additional implementation details, prompts, and hyperparameters are reported in~\cref{app:details}.

\paragraph{Baselines}
We compare against representative annotation routing baselines: ARAIDA~\citep{huang2024araida} (retrieval-augmented error-aware routing), CoAnnotating~\citep{li2023coannotating} and CDI~\citep{gligoric2025can} (confidence-based human--LLM collaboration), HyPAC~\citep{zeng2026hypac} (PAC-guaranteed budget-aware routing) and SANT~\citep{huang2024selective} (active learning with selective triage). We also insert the Random baseline, while we avoid inserting~\citet{wang2021want}, which was improved by~\citet{li2023coannotating} and~\citet{miranda2025hybrid} due to reproducibility problems. Since SANT and ARAIDA rely on label-based confidence estimation, they cannot be extended to summarization and are omitted in those settings.

\paragraph{Datasets}
We evaluate on several benchmarks spanning classification with AG's News~\cite{gulli2004ag}, SST2~\citep{socher-etal-2013-recursive}, IMDB~\citep{maas-EtAl:2011:ACL-HLT2011} and PubMed~\citep{dernoncourt2017pubmed}, multiple-choice QA with Global-MMLU~\citep{singh2025global} and MMLU-Redux~\cite{gema2025we} and summarization with CNN/DailyMail~\cite{see-etal-2017-get} and XLSum~\cite{hasan-etal-2021-xl}. Global-MMLU and XLSum are included to assess performance in multilingual settings.
\begin{table*}[t]
\centering
\caption{Mean $\pm$ std in the \textsc{auditor style} setting. For each dataset, we report the number of samples, Accuracy or ROUGE-1 depending on the task ($\uparrow$), Machine Cumulative Accuracy ($\uparrow$), and total annotation cost ($\downarrow$) for varying percentages of human-labeled data. \textbf{Bold} indicates the best-performing method, while \underline{underlined} the second best. Value in parentheses represents the average performance obtained by LLM annotators on that dataset.}\label{tab:auditor_style}
\resizebox{\linewidth}{!}{%
\begin{tabular}{llccc|ccc|ccc|ccc}
\toprule
\multirow{2}{*}{\textbf{Dataset}} & \multirow{2}{*}{\textbf{Method}} & \multicolumn{3}{c}{\textbf{10\%}} & \multicolumn{3}{c}{\textbf{30\%}} & \multicolumn{3}{c}{\textbf{50\%}} & \multicolumn{3}{c}{\textbf{70\%}} \\
\cmidrule(lr){3-5} \cmidrule(lr){6-8} \cmidrule(lr){9-11} \cmidrule(lr){12-14}
& & Accuracy/ROUGE & MCA & Cost & Accuracy/ROUGE & MCA & Cost & Accuracy/ROUGE & MCA & Cost & Accuracy/ROUGE & MCA & Cost \\
\midrule
\multirow{7}{*}{\shortstack{AG's News \\ (LLM: 0.834) \\ 7,600 samples}}
& Random & $0.729 \pm 0.001$ & $0.163 \pm 0.012$ & $144$ & $0.805 \pm 0.002$ & $0.146 \pm 0.005$ & $281$ & $0.794 \pm 0.002$ & $0.145 \pm 0.004$ & $418$ & $0.804 \pm 0.001$ & $0.154 \pm 0.002$ & $555$ \\
& ARAIDA & $0.787 \pm 0.010$ & $\mathbf{0.255 \pm 0.050}$ & $152$ & $0.834 \pm 0.013$ & $0.218 \pm 0.012$ & $304$ & $0.861 \pm 0.035$ & $0.217 \pm 0.014$ & $456$ & $0.886 \pm 0.055$ & $0.215 \pm 0.010$ & $608$ \\
& CDI & $0.800 \pm 0.001$ & $0.229 \pm 0.012$ & $152$ & $\underline{0.841 \pm 0.001}$ & $0.215 \pm 0.005$ & $304$ & $0.887 \pm 0.003$ & $0.222 \pm 0.005$ & $456$ & $0.921 \pm 0.002$ & $0.220 \pm 0.002$ & $608$ \\
& CoAnnotating & $\underline{0.800 \pm 0.001}$ & $\underline{0.233 \pm 0.001}$ & $152$ & $0.844 \pm 0.001$ & $0.223 \pm 0.001$ & $304$ & $0.889 \pm 0.001$ & $0.224 \pm 0.001$ & $456$ & $\underline{0.922 \pm 0.001}$ & $0.221 \pm 0.001$ & $608$ \\
& HyPAC & $0.794 \pm 0.007$ & $0.220 \pm 0.016$ & $152$ & $0.840 \pm 0.013$ & $\mathbf{0.226 \pm 0.007}$ & $304$ & $\underline{0.875 \pm 0.033}$ & $\underline{0.225 \pm 0.009}$ & $456$ & $0.903 \pm 0.052$ & $\underline{0.224 \pm 0.008}$ & $608$ \\
& SANT & $0.798 \pm 0.007$ & $0.236 \pm 0.026$ & $152$ & $0.838 \pm 0.009$ & $0.216 \pm 0.005$ & $304$ & $0.867 \pm 0.026$ & $0.216 \pm 0.006$ & $456$ & $0.896 \pm 0.041$ & $0.217 \pm 0.006$ & $608$ \\
& \approach & $\mathbf{0.800 \pm 0.001}$ & $0.224 \pm 0.010$ & $\mathbf{144}$ & $\mathbf{0.845 \pm 0.002}$ & $\underline{0.225 \pm 0.005}$ & $\mathbf{281}$ & $\mathbf{0.890 \pm 0.001}$ & $\mathbf{0.227 \pm 0.003}$ & $\mathbf{418}$ & $\mathbf{0.936 \pm 0.001}$ & $\mathbf{0.226 \pm 0.001}$ & $\mathbf{555}$ \\
\midrule
\multirow{7}{*}{\shortstack{SST2 \\ (LLM: 0.944) \\ 870 samples}}
& Random & $0.902 \pm 0.003$ & $0.067 \pm 0.027$ & $16$ & $0.921 \pm 0.003$ & $0.062 \pm 0.009$ & $32$ & $0.934 \pm 0.002$ & $0.062 \pm 0.004$ & $48$ & $0.943 \pm 0.003$ & $0.061 \pm 0.004$ & $64$ \\
& ARAIDA & $0.939 \pm 0.002$ & $0.047 \pm 0.020$ & $17$ & $0.946 \pm 0.003$ & $0.043 \pm 0.005$ & $35$ & $0.948 \pm 0.006$ & $0.044 \pm 0.007$ & $52$ & $0.949 \pm 0.009$ & $0.045 \pm 0.008$ & $70$ \\
& CDI & $0.941 \pm 0.003$ & $0.055 \pm 0.029$ & $17$ & $0.952 \pm 0.003$ & $0.054 \pm 0.010$ & $35$ & $\underline{0.970 \pm 0.004}$ & $0.069 \pm 0.009$ & $52$ & $\underline{0.985 \pm 0.002}$ & $0.070 \pm 0.004$ & $70$ \\
& CoAnnotating & $0.948 \pm 0.001$ & $0.126 \pm 0.001$ & $17$ & $\underline{0.959 \pm 0.001}$ & $\underline{0.077 \pm 0.001}$ & $35$ & $0.964 \pm 0.002$ & $0.057 \pm 0.005$ & $52$ & $0.978 \pm 0.002$ & $0.060 \pm 0.003$ & $70$ \\
& HyPAC & $\underline{0.966 \pm 0.007}$ & $\underline{0.173 \pm 0.035}$ & $17$ & $0.950 \pm 0.003$ & $0.048 \pm 0.011$ & $35$ & $0.969 \pm 0.003$ & $\mathbf{0.100 \pm 0.034}$ & $52$ & $0.969 \pm 0.003$ & $\underline{0.088 \pm 0.028}$ & $70$ \\
& SANT & $0.940 \pm 0.002$ & $0.045 \pm 0.018$ & $17$ & $0.951 \pm 0.003$ & $0.051 \pm 0.008$ & $35$ & $0.963 \pm 0.007$ & $0.056 \pm 0.007$ & $52$ & $0.975 \pm 0.011$ & $0.058 \pm 0.008$ & $70$ \\
& \approach & $\mathbf{0.968 \pm 0.001}$ & $\mathbf{0.188 \pm 0.006}$ & $\mathbf{16}$ & $\mathbf{0.969 \pm 0.003}$ & $\mathbf{0.124 \pm 0.032}$  & $\mathbf{32}$ & $\mathbf{0.971 \pm 0.001}$ & $\underline{0.071 \pm 0.001}$  & $\mathbf{48}$ & $\mathbf{0.990 \pm 0.001}$ & $\mathbf{0.097 \pm 0.001}$ & $\mathbf{64}$ \\
 \midrule
\multirow{7}{*}{\shortstack{MMLU-Redux \\ (LLM: 0.633) \\ 3,000 samples}}
& Random & $0.706 \pm 0.002$ & $0.204 \pm 0.021$ & $57$ & $0.691 \pm 0.003$ & $0.216 \pm 0.010$ & $111$ & $0.767 \pm 0.004$ & $0.196 \pm 0.007$ & $165$ & $0.901 \pm 0.004$ & $0.224 \pm 0.006$ & $219$ \\
& ARAIDA & $0.749 \pm 0.011$ & $0.278 \pm 0.018$ & $60$ & $0.794 \pm 0.013$ & $0.265 \pm 0.008$ & $120$ & $0.811 \pm 0.036$ & $0.266 \pm 0.006$ & $180$ & $0.826 \pm 0.059$ & $0.266 \pm 0.006$ & $240$ \\
& CoAnnotating & $\underline{0.772 \pm 0.001}$ & $\underline{0.430 \pm 0.001}$ & $60$ & $\mathbf{0.836 \pm 0.001}$ & $\mathbf{0.357 \pm 0.001}$ & $120$ & $\underline{0.889 \pm 0.001}$ & $\underline{0.321 \pm 0.001}$ & $180$ & $\underline{0.939 \pm 0.001}$ & $\underline{0.300 \pm 0.001}$ & $240$ \\
& HyPAC & $0.764 \pm 0.003$ & $0.355 \pm 0.026$ & $60$ & $\underline{0.832 \pm 0.003}$ & $\underline{0.344 \pm 0.009}$ & $120$ & $0.884 \pm 0.003$ & $0.311 \pm 0.006$ & $180$ & $0.936 \pm 0.003$ & $0.297 \pm 0.002$ & $240$ \\
& CDI & $0.759 \pm 0.003$ & $0.307 \pm 0.029$ & $60$ & $0.808 \pm 0.003$ & $0.270 \pm 0.009$ & $120$ & $0.865 \pm 0.004$ & $0.275 \pm 0.008$ & $180$ & $0.923 \pm 0.005$ & $0.279 \pm 0.006$ & $240$ \\
& SANT & $0.754 \pm 0.008$ & $0.278 \pm 0.018$ & $60$ & $0.807 \pm 0.009$ & $0.271 \pm 0.006$ & $120$ & $0.854 \pm 0.027$ & $0.275 \pm 0.007$ & $180$ & $0.899 \pm 0.046$ & $0.274 \pm 0.006$ & $240$ \\
& \approach & $\mathbf{0.778 \pm 0.002}$ & $\mathbf{0.497 \pm 0.023}$ & $\mathbf{57}$ & $0.815 \pm 0.002$ & $0.287 \pm 0.007$ & $\mathbf{111}$ & $\mathbf{0.891 \pm 0.002}$ & $\mathbf{0.373 \pm 0.003}$ & $\mathbf{165}$ & $\mathbf{0.954 \pm 0.002}$ & $\mathbf{0.364 \pm 0.003}$ & $\mathbf{219}$ \\
\midrule
\multirow{7}{*}{\shortstack{IMDB \\ (LLM: 0.954) \\ 25,000 samples}}
& Random & $0.901 \pm 0.001$ & $0.016 \pm 0.006$ & $475$ & $0.912 \pm 0.001$ & $0.015 \pm 0.002$ & $925$ & $0.923 \pm 0.001$ & $0.025 \pm 0.001$ & $1375$ & $0.944 \pm 0.001$ & $0.022 \pm 0.001$ & $1825$ \\
& ARAIDA & $0.948 \pm 0.002$ & $0.080 \pm 0.024$ & $500$ & $0.960 \pm 0.003$ & $\underline{0.057 \pm 0.001}$ & $1000$ & $0.965 \pm 0.008$ & $0.056 \pm 0.002$ & $1500$ & $0.970 \pm 0.013$ & $0.056 \pm 0.002$ & $2000$ \\
& CDI & $\underline{0.946 \pm 0.001}$ & $\underline{0.100 \pm 0.004}$ & $500$ & $0.961 \pm 0.001$ & $0.052 \pm 0.001$ & $1000$ & $0.969 \pm 0.001$ & $0.046 \pm 0.001$ & $1500$ & $0.985 \pm 0.001$ & $0.056 \pm 0.001$ & $2000$ \\
& CoAnnotating & $0.946 \pm 0.001$ & $0.006 \pm 0.001$ & $500$ & $0.960 \pm 0.001$ & $0.049 \pm 0.001$ & $1000$ & $\mathbf{0.981 \pm 0.001}$ & $\mathbf{0.071 \pm 0.001}$ & $1500$ & $\underline{0.990 \pm 0.001}$ & $\underline{0.064 \pm 0.001}$ & $2000$ \\
& HyPAC & $0.947 \pm 0.001$ & $0.058 \pm 0.012$ & $500$ & $0.950 \pm 0.001$ & $0.055 \pm 0.005$ & $1000$ & $0.954 \pm 0.001$ & $0.055 \pm 0.004$ & $1500$ & $0.957 \pm 0.001$ & $0.054 \pm 0.002$ & $2000$ \\
& SANT & $0.946 \pm 0.001$ & $0.100 \pm 0.001$ & $500$ & $0.957 \pm 0.001$ & $0.057 \pm 0.001$ & $1000$ & $0.957 \pm 0.001$ & $\underline{0.057 \pm 0.001}$ & $1500$ & $0.957 \pm 0.001$ & $0.057 \pm 0.001$ & $2000$ \\
& \approach & $\mathbf{0.951 \pm 0.001}$ & $\mathbf{0.056 \pm 0.002}$ & $\mathbf{475}$ & $\mathbf{0.962 \pm 0.001}$ & $\mathbf{0.058 \pm 0.001}$ & $\mathbf{925}$ & $\underline{0.973 \pm 0.001}$ & $0.056 \pm 0.001$ & $\mathbf{1375}$ & $\mathbf{0.993 \pm 0.001}$ & $\mathbf{0.064 \pm 0.001}$ & $\mathbf{1825}$ \\ \midrule
\multirow{7}{*}{\shortstack{PubMed \\ (LLM: 0.701) \\ 30,000 samples}}
& Random & $0.616 \pm 0.001$ & $0.317 \pm 0.008$ & $562$ & $0.708 \pm 0.001$ & $0.279 \pm 0.004$ & $1094$ & $0.731 \pm 0.002$ & $0.250 \pm 0.004$ & $1627$ & $0.752 \pm 0.001$ & $0.259 \pm 0.002$ & $2159$ \\
& ARAIDA & $0.666 \pm 0.019$ & $\underline{0.410 \pm 0.017}$ & $591$ & $0.728 \pm 0.018$ & $\underline{0.368 \pm 0.005}$ & $1183$ & $0.748 \pm 0.049$ & $0.364 \pm 0.002$ & $1775$ & $0.762 \pm 0.072$ & $0.362 \pm 0.006$ & $2366$ \\
& CDI & $0.677 \pm 0.001$ & $0.362 \pm 0.008$ & $591$ & $0.749 \pm 0.001$ & $0.362 \pm 0.003$ & $1183$ & $0.820 \pm 0.001$ & $0.360 \pm 0.002$ & $1775$ & $0.892 \pm 0.001$ & $0.360 \pm 0.002$ & $2366$ \\
& CoAnnotating & $\underline{0.678 \pm 0.001}$ & $0.375 \pm 0.001$ & $591$ & $\underline{0.750 \pm 0.001}$ & $0.364 \pm 0.001$ & $1183$ & $\underline{0.821 \pm 0.001}$ & $0.361 \pm 0.001$ & $1775$ & $\underline{0.890 \pm 0.001}$ & $0.356 \pm 0.001$ & $2366$ \\
& HyPAC & $0.678 \pm 0.001$ & $0.370 \pm 0.005$ & $591$ & $0.749 \pm 0.001$ & $0.361 \pm 0.003$ & $1183$ & $0.820 \pm 0.001$ & $0.359 \pm 0.002$ & $1775$ & $0.891 \pm 0.001$ & $0.358 \pm 0.002$ & $2366$ \\
& SANT & $0.673 \pm 0.011$ & $0.366 \pm 0.036$ & $591$ & $0.714 \pm 0.001$ & $0.365 \pm 0.001$ & $1183$ & $0.714 \pm 0.001$ & $\underline{0.364 \pm 0.001}$ & $1775$ & $0.714 \pm 0.001$ & $\underline{0.365 \pm 0.001}$ & $2366$ \\
& \approach & $\mathbf{0.683 \pm 0.001}$ & $\mathbf{0.424 \pm 0.011}$ & $\mathbf{562}$ & $\mathbf{0.762 \pm 0.001}$ & $\mathbf{0.404 \pm 0.003}$ & $\mathbf{1094}$ & $\mathbf{0.835 \pm 0.001}$ & $\mathbf{0.389 \pm 0.001}$ & $\mathbf{1627}$ & $\mathbf{0.905 \pm 0.001}$ & $\mathbf{0.377 \pm 0.001}$ & $\mathbf{2159}$ \\ \midrule
\multirow{5}{*}{\shortstack{CNN \\ (LLM: 0.363) \\ 11,490 Samples}}
& Random & $0.402 \pm 0.001$ & - & 10455  & $0.521 \pm 0.001$ & - & 31137 & $0.619 \pm 0.001$ & - & 51819 & $0.781 \pm 0.001$ & - & 72501 \\
& CDI & $0.423 \pm 0.001$ & - & 11604  & $\underline{0.556 \pm 0.001}$ & - & 34584 & $\mathbf{0.682 \pm 0.001}$ & - & 57564 & $0.807 \pm 0.004$ & - & 80544 \\
& CoAnnotating & $\underline{0.423 \pm 0.001}$ & - & 11604 & $0.550 \pm 0.001$ & - & 34584 & $0.677 \pm 0.001$ & - & 57564  & $\underline{0.810 \pm 0.001}$ & - & 80544 \\
& HyPAC & $0.423 \pm 0.001$ & - & 11604 & $0.550 \pm 0.001$ & - & 34584 & $0.677 \pm 0.001$ & - & 57564  & $0.806 \pm 0.001$ & - & 80544 \\
& \approach & $\mathbf{0.425 \pm 0.001}$ & - & \textbf{10455} & $\mathbf{0.560 \pm 0.001}$ & - & \textbf{31137} & $\underline{0.678 \pm 0.001}$ & - & \textbf{51819} & $\mathbf{0.816 \pm 0.001}$ & - & \textbf{72501} \\
\bottomrule
\end{tabular}}
\end{table*}

\paragraph{Experimental Settings}
We study two regimes:
\begin{itemize}
    \item \textsc{Auditor Style}: routing under a fixed budget of human annotations, deciding between human and LLM supervision.
    \item \textsc{Dollars}: routing under monetary constraints with heterogeneous annotator costs, allowing selection among multiple annotators.
\end{itemize}

We use \href{https://www.anthropic.com/news/claude-sonnet-4-5}{Claude Sonnet 4.5}, \href{https://aws.amazon.com/nova/}{Nova Pro v1.0}, and \href{https://huggingface.co/Qwen/Qwen3-32B}{Qwen3-32B} as annotators. We assume heterogeneous annotation, reflecting the pricing of Amazon Bedrock\footnote{\url{aws.amazon.com/bedrock/pricing}} models as of May 2026. We assign a cost of \$0.05 per annotation for Claude, \$0.03 for Nova Pro, \$0.01 for Qwen, and \$0.10\footnote{\url{https://aitaggers.com.au/pricing}} for human annotations. The low human cost reflects the simplicity of the task. This represents a conservative setting for \quorum: LLM costs are invariant to task difficulty, while human costs scale with complexity, so the savings from LLM-based annotation would only grow for harder tasks.
\quorum does not rely on uncertainty estimation, which requires additional LLM inference and therefore incurs extra annotation cost. For uncertainty-based methods, we reproduce prior settings using self-verbalized confidence~\citep{gligoric2025can}. Prompt templates and details are reported in Appendix~\ref{app:prompts}.

\begin{table*}[!t]
\centering
\caption{Mean $\pm$ std in the \textsc{dollars} setting. For each dataset, we report the full human annotation cost, Accuracy or ROUGE-1 depending on the task ($\uparrow$), and Machine Cumulative Accuracy ($\uparrow$) under varying budget constraints. \textbf{Bold} indicates the best-performing method, \underline{underlined} the second best. Row values indicate the available annotation budget, expressed as a percentage of the cost required to annotate the entire dataset using human annotators.}
\label{tab:dollars}
\resizebox{\linewidth}{!}{%
\begin{tabular}{llcc|cc|cc|cc}
\toprule
\multirow{2}{*}{\textbf{Dataset}} & \multirow{2}{*}{\textbf{Method}} & \multicolumn{2}{c}{\textbf{10\%}} & \multicolumn{2}{c}{\textbf{30\%}} & \multicolumn{2}{c}{\textbf{50\%}} & \multicolumn{2}{c}{\textbf{70\%}} \\
\cmidrule(lr){3-4} \cmidrule(lr){5-6} \cmidrule(lr){7-8} \cmidrule(lr){9-10}
& & Accuracy/ROUGE & MCA & Accuracy/ROUGE & MCA  & Accuracy/ROUGE & MCA & Accuracy/ROUGE & MCA \\ \midrule
\multirow{7}{*}{\shortstack{AG's News \\ (LLM: 0.834) \\ Human Cost: 760}}
& Random & $0.678 \pm 0.001$ & $0.059 \pm 0.070$ & $0.679 \pm 0.001$ & $0.064 \pm 0.050$ & $0.681 \pm 0.001$ & $0.050 \pm 0.028$ & $0.688 \pm 0.001$ & $0.064 \pm 0.013$ \\
& ARAIDA & $0.778 \pm 0.001$ & $\mathbf{0.120 \pm 0.001}$ & $0.783 \pm 0.012$ & $0.060 \pm 0.001$ & $0.786 \pm 0.018$ & $\mathbf{0.110 \pm 0.001}$ & $0.792 \pm 0.030$ & $\mathbf{0.095 \pm 0.001}$ \\
& CDI & $0.778 \pm 0.001$ & $\underline{0.120 \pm 0.001}$ & $0.779 \pm 0.001$ & $0.060 \pm 0.001$ & $0.780 \pm 0.001$ & $\underline{0.110 \pm 0.001}$ & $0.783 \pm 0.001$ & $\underline{0.095 \pm 0.001}$ \\
& CoAnnotating & $0.778 \pm 0.001$ & $0.069 \pm 0.001$ & $0.778 \pm 0.001$ & $0.069 \pm 0.001$ & $0.778 \pm 0.001$ & $0.069 \pm 0.001$ & $0.778 \pm 0.001$ & $0.069 \pm 0.001$ \\
& HyPAC & $\underline{0.783 \pm 0.003}$ & $0.100 \pm 0.017$ & $\underline{0.785 \pm 0.005}$ & $\underline{0.094 \pm 0.016}$ & $\underline{0.790 \pm 0.008}$ & $0.097 \pm 0.015$ & $\underline{0.797 \pm 0.013}$ & $0.094 \pm 0.010$ \\
& SANT & $0.781 \pm 0.009$ & $0.120 \pm 0.001$ & $0.783 \pm 0.012$ & $0.060 \pm 0.001$ & $0.786 \pm 0.018$ & $0.110 \pm 0.001$ & $0.792 \pm 0.030$ & $0.095 \pm 0.001$ \\
& \approach & $\mathbf{0.794 \pm 0.001}$ & $0.084 \pm 0.001$ & $\mathbf{0.809 \pm 0.005}$ & $\mathbf{0.099 \pm 0.001}$ & $\mathbf{0.829 \pm 0.010}$ & $0.094 \pm 0.001$ & $\mathbf{0.887 \pm 0.028}$ & $0.085 \pm 0.003$ \\
\midrule
\multirow{7}{*}{\shortstack{SST2 \\ (LLM: 0.944) \\ Human Cost: 87}}
& Random & $0.906 \pm 0.002$ & $0.015 \pm 0.091$ & $0.916 \pm 0.001$ & $0.020 \pm 0.089$ & $0.928 \pm 0.001$ & $0.038 \pm 0.061$ & $0.940 \pm 0.002$ & $0.040 \pm 0.027$ \\
& ARAIDA & $0.936 \pm 0.002$ & $0.000 \pm 0.002$ & $0.937 \pm 0.002$ & $0.000 \pm 0.002$ & $0.936 \pm 0.002$ & $0.000 \pm 0.002$ & $0.937 \pm 0.002$ & $0.045 \pm 0.002$ \\
& CDI & $0.936 \pm 0.002$ & $0.000 \pm 0.002$ & $0.936 \pm 0.002$ & $0.000 \pm 0.003$ & $0.936 \pm 0.003$ & $0.000 \pm 0.003$ & $0.937 \pm 0.003$ & $\underline{0.045 \pm 0.003}$ \\
& CoAnnotating & $0.937 \pm 0.003$ & $0.000 \pm 0.003$ & $0.937 \pm 0.003$ & $0.000 \pm 0.003$ & $0.937 \pm 0.003$ & $0.000 \pm 0.003$ & $0.937 \pm 0.003$ & $0.000 \pm 0.003$ \\
& HyPAC & $\underline{0.938 \pm 0.004}$ & $\underline{0.016 \pm 0.024}$ & $\underline{0.937 \pm 0.001}$ & $\underline{0.020 \pm 0.037}$ & $\underline{0.940 \pm 0.003}$ & $\underline{0.051 \pm 0.030}$ & $\underline{0.941 \pm 0.003}$ & $0.036 \pm 0.016$ \\
& SANT & $0.936 \pm 0.002$ & $0.000 \pm 0.003$ & $0.936 \pm 0.001$ & $0.000 \pm 0.001$ & $0.937 \pm 0.003$ & $0.000 \pm 0.001$ & $0.939 \pm 0.005$ & $0.045 \pm 0.001$ \\
& \approach & $\mathbf{0.944 \pm 0.001}$ & $\mathbf{0.021 \pm 0.001}$ & $\mathbf{0.950 \pm 0.001}$ & $\mathbf{0.022 \pm 0.001}$ & $\mathbf{0.953 \pm 0.001}$ & $\mathbf{0.054 \pm 0.001}$ & $\mathbf{0.967 \pm 0.001}$ & $\mathbf{0.048 \pm 0.001}$ \\
\midrule
\multirow{7}{*}{\shortstack{MMLU-Redux \\ (LLM: 0.633) \\ Human Cost: 300}}
& Random & $0.630 \pm 0.001$ & $0.370 \pm 0.142$ & $0.631 \pm 0.001$ & $0.379 \pm 0.072$ & $0.634 \pm 0.001$ & $0.405 \pm 0.051$ & $0.671 \pm 0.002$ & $0.383 \pm 0.032$ \\
& ARAIDA & $0.730 \pm 0.001$ & $0.111 \pm 0.001$ & $0.731 \pm 0.001$ & $0.421 \pm 0.002$ & $0.732 \pm 0.002$ & $0.410 \pm 0.002$ & $0.734 \pm 0.002$ & $0.329 \pm 0.002$ \\
& CDI & $0.730 \pm 0.002$ & $0.111 \pm 0.002$ & $0.731 \pm 0.002$ & $0.421 \pm 0.002$ & $0.732 \pm 0.002$ & $0.410 \pm 0.002$ & $0.734 \pm 0.002$ & $0.329 \pm 0.002$ \\
& CoAnnotating & $0.595 \pm 0.002$ & $0.336 \pm 0.002$ & $0.626 \pm 0.002$ & $\mathbf{0.482 \pm 0.002}$ & $0.686 \pm 0.002$ & $\mathbf{0.542 \pm 0.002}$ & $\mathbf{0.795 \pm 0.002}$ & $\mathbf{0.543 \pm 0.002}$ \\
& HyPAC & $\underline{0.735 \pm 0.001}$ & $\mathbf{0.401 \pm 0.048}$ & $\underline{0.737 \pm 0.001}$ & $\underline{0.428 \pm 0.043}$ & $\underline{0.741 \pm 0.001}$ & $0.427 \pm 0.044$ & $0.749 \pm 0.002$ & $0.417 \pm 0.037$ \\
& SANT & $0.730 \pm 0.002$ & $0.111 \pm 0.002$ & $0.731 \pm 0.002$ & $0.421 \pm 0.002$ & $0.732 \pm 0.001$ & $0.410 \pm 0.001$ & $0.734 \pm 0.001$ & $0.329 \pm 0.001$ \\
& \approach & $\mathbf{0.735 \pm 0.001}$ & $\underline{0.399 \pm 0.012}$ & $\mathbf{0.747 \pm 0.001}$ & $0.389 \pm 0.181$ & $\mathbf{0.759 \pm 0.001}$ & $\underline{0.492 \pm 0.058}$ & $\underline{0.780 \pm 0.001}$ & $\underline{0.439 \pm 0.015}$ \\
\midrule
\multirow{7}{*}{\shortstack{IMDB \\ (LLM: 0.954) \\ Human Cost: 2,500}}
& Random & $0.916 \pm 0.003$ & $0.053 \pm 0.008$ & $0.927 \pm 0.003$ & $0.056 \pm 0.011$ & $0.937 \pm 0.003$ & $0.053 \pm 0.008$ & $0.928 \pm 0.003$ & $0.037 \pm 0.005$ \\
& ARAIDA & $0.946 \pm 0.001$ & $0.110 \pm 0.001$ & $0.946 \pm 0.001$ & $0.097 \pm 0.001$ & $0.947 \pm 0.001$ & $0.097 \pm 0.001$ & $0.947 \pm 0.001$ & $0.071 \pm 0.001$ \\
& CDI & $0.946 \pm 0.001$ & $\underline{0.110 \pm 0.001}$ & $0.946 \pm 0.001$ & $\underline{0.097 \pm 0.001}$ & $0.947 \pm 0.001$ & $\underline{0.097 \pm 0.001}$ & $0.947 \pm 0.001$ & $\underline{0.071 \pm 0.001}$ \\
& CoAnnotating & $\underline{0.947 \pm 0.001}$ & $0.026 \pm 0.001$ & $\underline{0.947 \pm 0.001}$ & $0.014 \pm 0.001$ & $0.947 \pm 0.001$ & $0.007 \pm 0.001$ & $0.947 \pm 0.003$ & $0.007 \pm 0.003$ \\
& HyPAC & $0.947 \pm 0.003$ & $0.040 \pm 0.006$ & $0.947 \pm 0.003$ & $0.035 \pm 0.004$ & $\underline{0.948 \pm 0.001}$ & $0.038 \pm 0.008$ & $\underline{0.950 \pm 0.001}$ & $0.039 \pm 0.005$ \\
& SANT & $0.946 \pm 0.003$ & $0.110 \pm 0.003$ & $0.946 \pm 0.003$ & $0.097 \pm 0.003$ & $0.947 \pm 0.003$ & $0.097 \pm 0.003$ & $0.947 \pm 0.002$ & $0.071 \pm 0.002$ \\
& \approach & $\mathbf{0.949 \pm 0.002}$ & $\mathbf{0.118 \pm 0.040}$ & $\mathbf{0.952 \pm 0.001}$ & $\mathbf{0.124 \pm 0.003}$ & $\mathbf{0.960 \pm 0.003}$ & $\mathbf{0.160 \pm 0.002}$ & $\mathbf{0.968 \pm 0.006}$ & $\mathbf{0.153 \pm 0.001}$ \\
\midrule
\multirow{7}{*}{\shortstack{PubMed \\ (LLM: 0.701) \\ Human Cost: 3,000}}
& Random & $0.342 \pm 0.001$ & $0.174 \pm 0.036$ & $0.614 \pm 0.001$ & $0.097 \pm 0.025$ & $0.627 \pm 0.001$ & $0.113 \pm 0.015$ & $0.638 \pm 0.001$ & $0.172 \pm 0.010$ \\
& ARAIDA & $0.647 \pm 0.016$ & $0.206 \pm 0.002$ & $0.650 \pm 0.021$ & $0.175 \pm 0.002$ & $0.646 \pm 0.002$ & $\underline{0.192 \pm 0.002}$ & $0.651 \pm 0.001$ & $0.217 \pm 0.002$ \\
& CDI & $0.642 \pm 0.002$ & $\underline{0.206 \pm 0.002}$ & $0.643 \pm 0.002$ & $0.175 \pm 0.002$ & $0.646 \pm 0.002$ & $0.192 \pm 0.002$ & $0.651 \pm 0.001$ & $\underline{0.217 \pm 0.001}$ \\
& CoAnnotating & $\underline{0.655 \pm 0.003}$ & $0.182 \pm 0.001$ & $\underline{0.655 \pm 0.003}$ & $0.182 \pm 0.001$ & $0.655 \pm 0.002$ & $0.182 \pm 0.001$ & $0.656 \pm 0.001$ & $0.182 \pm 0.001$ \\
& HyPAC & $0.650 \pm 0.006$ & $0.192 \pm 0.015$ & $0.654 \pm 0.007$ & $\underline{0.197 \pm 0.013}$ & $\underline{0.660 \pm 0.011}$ & $0.188 \pm 0.015$ & $\underline{0.673 \pm 0.018}$ & $0.188 \pm 0.008$ \\
& SANT & $0.647 \pm 0.016$ & $0.206 \pm 0.001$ & $0.650 \pm 0.021$ & $0.175 \pm 0.001$ & $0.648 \pm 0.006$ & $\mathbf{0.192 \pm 0.001}$ & $0.664 \pm 0.039$ & $0.217 \pm 0.001$ \\
& \approach & $\mathbf{0.660 \pm 0.004}$ & $\mathbf{0.278 \pm 0.002}$ & $\mathbf{0.678 \pm 0.009}$ & $\mathbf{0.272 \pm 0.001}$ & $\mathbf{0.714 \pm 0.019}$ & $0.176 \pm 0.001$ & $\mathbf{0.787 \pm 0.039}$ & $\mathbf{0.242 \pm 0.002}$ \\
\midrule
\multirow{5}{*}{\shortstack{CNN \\ (LLM: 0.363) \\ Human Cost: 1,149}}
& Random & $0.342 \pm 0.001$ & - &  $0.295 \pm 0.001$ & - &  $0.341 \pm 0.001$ & - & $0.352 \pm 0.001$ & -  \\
& CDI & $0.361 \pm 0.001$ & - & $0.364 \pm 0.001$ & - &  $0.368 \pm 0.001$ & - & $0.378 \pm 0.001$  & - \\
& CoAnnotating  & $\underline{0.391 \pm 0.001}$ & - &  $\underline{0.422 \pm 0.001}$ & - & $\underline{0.486 \pm 0.001}$ & - &  $\underline{0.613 \pm 0.001}$ & - \\
& HyPAC  &  $0.373 \pm 0.001$ & - &  $0.378 \pm 0.001$ & - &  $0.388 \pm 0.001$ & - &  $0.407 \pm 0.001$  & - \\
& \approach & $\mathbf{0.476 \pm 0.001}$ & - &  $\mathbf{0.491 \pm 0.001}$ & - & $\mathbf{0.501 \pm 0.002}$ & - &  $\mathbf{0.638 \pm 0.002}$ & - \\
\bottomrule
\end{tabular}}
\end{table*}

\paragraph{Metrics} Following prior work~\citep{huang2024araida}, we report Accuracy for classification and ROUGE-1 for summarization. We evaluate Machine Cumulative Accuracy (MCA)~\citep{huang2024araida}, measuring how often instances routed to humans correspond to cases where LLM predictions would fail:
\begin{equation*}
\text{MCA} = \frac{\sum_{i \in H} \mathbbm1[\hat{y}_i \neq y_i]} {|H|}
\end{equation*}
Higher MCA indicates more effective allocation of humans. MCA is computed using Claude as the reference LLM, being the most expensive annotator. MCA is omitted for summarization tasks. We also report total annotation cost.

\paragraph{Implementation Details} We run all experiments using 10 random seeds for classification and 3 for summarization and report the mean and standard deviation across runs. Hyperparameters are selected based on sensitivity analysis (reported in~\cref{app:sensitivity}), using the configurations that achieved the best overall performance. We set the human escalation threshold to $\tau_H = 0.5$ and the parameters $\epsilon_0 = 0.3$ and $\gamma = 0.005$. 
\quorum is initialized with $\mu_a = \mathbf{0}$ and $\Sigma_a = I$ for LLM annotators, while human annotators are assigned a prior $\mu_H = 0.3 \cdot \mathbf{1}$ and $\Sigma_H = 0.5 I$. The maximum number of annotations per sample is set to 3, with a minimum of 1. The agreement threshold for consensus-based stopping is 0.66. 

While annotation complexity is not only determined by the characteristics of the input instance~\citep{cabitza2023toward}, but may also depend on annotator-specific factors, such as demographics and individual perspectives~\citep{sap2022annotators}, following previous research~\citep{gligoric2025can}, we use the human-provided labels from the original datasets as gold-standard annotations for evaluation purposes.

\section{Results}
We investigate the following research questions:

\noindent\textbf{RQ1.} Does \quorum improve annotation quality under limited annotation budgets?

\noindent\textbf{RQ2.} Can feature-based difficulty estimation replace uncertainty without sacrificing performance?

\noindent\textbf{RQ3.} Does allocating multiple annotations improve final prediction quality?

\noindent\textbf{RQ4.} Does \quorum transfer across languages?

\subsection*{RQ1. Budget-Aware Routing Evaluation}
We evaluate \quorum under two routing settings: \textsc{Auditor Style} and \textsc{Dollars}.

\paragraph{\textsc{Auditor Style}}
\Cref{tab:auditor_style} reports performance under varying human annotation budgets. Overall, \quorum achieves the strongest cost--performance trade-off, matching or improving performance while reducing annotation cost across datasets. The total cost for each sample is computed as the sum of per-call costs incurred for each annotator queried.
Improvements are most pronounced in low-budget regimes (10\%--30\%), where efficient allocation of human labels is critical. Gains are especially evident on challenging benchmarks such as PubMed and MMLU-Redux.
For summarization tasks, \quorum obtains the highest ROUGE-1 scores while maintaining the lowest annotation cost, showing generalization beyond classification settings. In a few cases at 10\% budget, performance falls below the average standalone LLM score (in parentheses). This is expected, as the LLM baseline reflects average annotator performance, whereas under strict budgets the system relies on the cheapest model.

\paragraph{\textsc{Dollars}}
\Cref{tab:dollars} evaluates routing under fixed monetary constraints with heterogeneous annotator costs. In this setting, \quorum benefits from increased budgets, while several baselines plateau early, indicating limited adaptability as additional resources become available.
This effect is evident on a challenging benchmark such as PubMed, where \quorum shows larger gains between low- and high-budget regimes. These results highlight more efficient budget allocation through selective rather than uniform use of expensive annotators.
For summarization, \quorum achieves the strongest ROUGE-1 scores, confirming that adaptive routing remains effective even when supervision quality is estimated via semantic similarity. An analysis of feature usage within the bandit is deferred to Appendix~\ref{app:feature_routing} due to space constraints.

\subsection*{RQ2. Features vs Uncertainty}
We compare our feature-based difficulty estimation with uncertainty-driven routing strategies. As shown in~\cref{tab:feature_importance}, our approach matches or outperforms uncertainty-based methods across all datasets and budget regimes. Crucially, uncertainty estimation requires additional LLM forward passes, introducing extra cost and computational overhead, whereas our features rely on lightweight, deterministic computations over the input text with no model inference. We compare against a random feature baseline, which performs worse, confirming that the improvement stems from the structured linguistic and semantic signals in our representation rather than feature dimensionality alone. Runtime comparisons in~\cref{app:time_analysis} further show that our approach is not only more economical in terms of annotation cost, but also substantially faster than confidence estimation or model retraining procedures, improving scalability on larger datasets.

\begin{table}[t]
    \centering
    \caption{Mean $\pm$ 95\% confidence intervals accuracy ($\uparrow$) with uncertainty-based routing, routing based on random features and our feature-based difficulty estimation. Budget is expressed as a percentage of dataset size.}
    \label{tab:feature_importance}
    \resizebox{\columnwidth}{!}{
    \begin{tabular}{c|cccc}
    \toprule
    \textbf{Budget} & \textbf{Dataset} & \textbf{LLM-Uncertainty Features} & \textbf{Random Features} & \textbf{Ours}  \\
    \midrule
    \multirow{3}{*}{10\%}
    & AG's News & $0.800 \pm 0.002$ & $0.697 \pm 0.009$ & $\mathbf{0.800 \pm 0.001}$ \\
    & SST2 & $0.948 \pm 0.001$ & $0.901 \pm 0.012$ & $\mathbf{0.968 \pm 0.001}$ \\
    & Global-MMLU (es) & $0.741 \pm 0.001$ & $0.518 \pm 0.012$ & $\mathbf{0.749 \pm 0.002}$ \\
    \midrule
    \multirow{3}{*}{70\%}
    & AG's News & $0.932 \pm 0.002$ & $0.712 \pm 0.007$ &  $\mathbf{0.936 \pm 0.001}$ \\
    & SST2 & $0.978 \pm 0.002$ & $0.921 \pm 0.010$ & $\mathbf{0.990 \pm 0.001}$ \\
    & Global-MMLU (es) & $0.910 \pm 0.001$ & $0.692 \pm 0.009$ & $\mathbf{0.912 \pm 0.001}$ \\
    \bottomrule
    \end{tabular}}
\end{table}

\subsection*{RQ3. Impact of Multiple Annotations}
\begin{table}[t]
    \centering
    \caption{Effect of multiple annotations per instance on \quorum's accuracy/ROUGE-1 across datasets.}
    \label{tab:multiple_annotations}
    \resizebox{\columnwidth}{!}{
    \begin{tabular}{c|cccc}
    \toprule
    \textbf{\#Annotations} & \textbf{AG's News}  & \textbf{PubMed} & \textbf{MMLU-Redux} & \textbf{CNN}\\
    \midrule
    1 & 0.852 & 0.740 & 0.721 & 0.387 \\
    2 & 0.873 & 0.756 & 0.754 & 0.398 \\
    3 & \textbf{0.891} & \textbf{0.791} & \textbf{0.766} & \textbf{0.410} \\
    \bottomrule
    \end{tabular}}
\end{table}

~\Cref{tab:multiple_annotations} evaluates the impact of increasing the number of annotations per instance on downstream performance in the \textsc{dollars} setting. Across all datasets, we observe an improvement as the number of annotations increases from one to three, confirming that additional supervision signals provide complementary information that improves label quality.
The gains are particularly pronounced on the CNN dataset, where performance increases from 0.387 to 0.410, suggesting that ambiguous instances benefit more from redundancy in annotations. These results support the choice of \quorum to dynamically allocate multiple annotations when disagreement is high, as additional labels translate into higher-quality final predictions.

To assess \quorum under a more realistic annotation setting, we conduct an additional experiment on the  \href{https://huggingface.co/datasets/rbnuria/SentiMP-En}{SentiMP-En} dataset~\citep{rodriguez2024federated}, where each textual sample is associated with three independent human annotations and one gold label. The gold labels are used only for evaluation, while the individual human annotations are used by the routing methods, naturally introducing annotator disagreement and noise. With a budget of 200 human annotations in the \textsc{Auditor Style} setting, \quorum consistently outperforms all competing routing strategies across all three annotators, while substantially improving over the standalone LLM baseline (60.4\% accuracy). These results suggest that \quorum remains effective under realistic, noisy, and heterogeneous human supervision.

\begin{table}[t]
    \centering
    \caption{Accuracy of different routing strategies on SentiMP-En under the \textsc{Auditor Style} setting with a budget of 200 human annotations. The noise rate is reported for each annotator. Best results for each annotator are shown in bold.}
    \label{tab:sentimp}
    \resizebox{\linewidth}{!}{
    \begin{tabular}{lcccc}
        \toprule
        \textbf{Labeller} & \textbf{QUORUM} & \textbf{CDI} & \textbf{CoAnnotating} & \textbf{Random} \\
        \midrule
        Labeller 1 (8\%)  & \textbf{0.7458} & 0.7292 & 0.7375 & 0.7313 \\
        Labeller 2 (8\%)  & \textbf{0.7500} & 0.7271 & 0.7250 & 0.7375 \\
        Labeller 3 (17\%) & \textbf{0.7271} & 0.6854 & 0.7125 & 0.6917 \\
        \bottomrule
    \end{tabular}}
\end{table}

\begin{figure}
    \centering
    \includegraphics[width=\columnwidth]{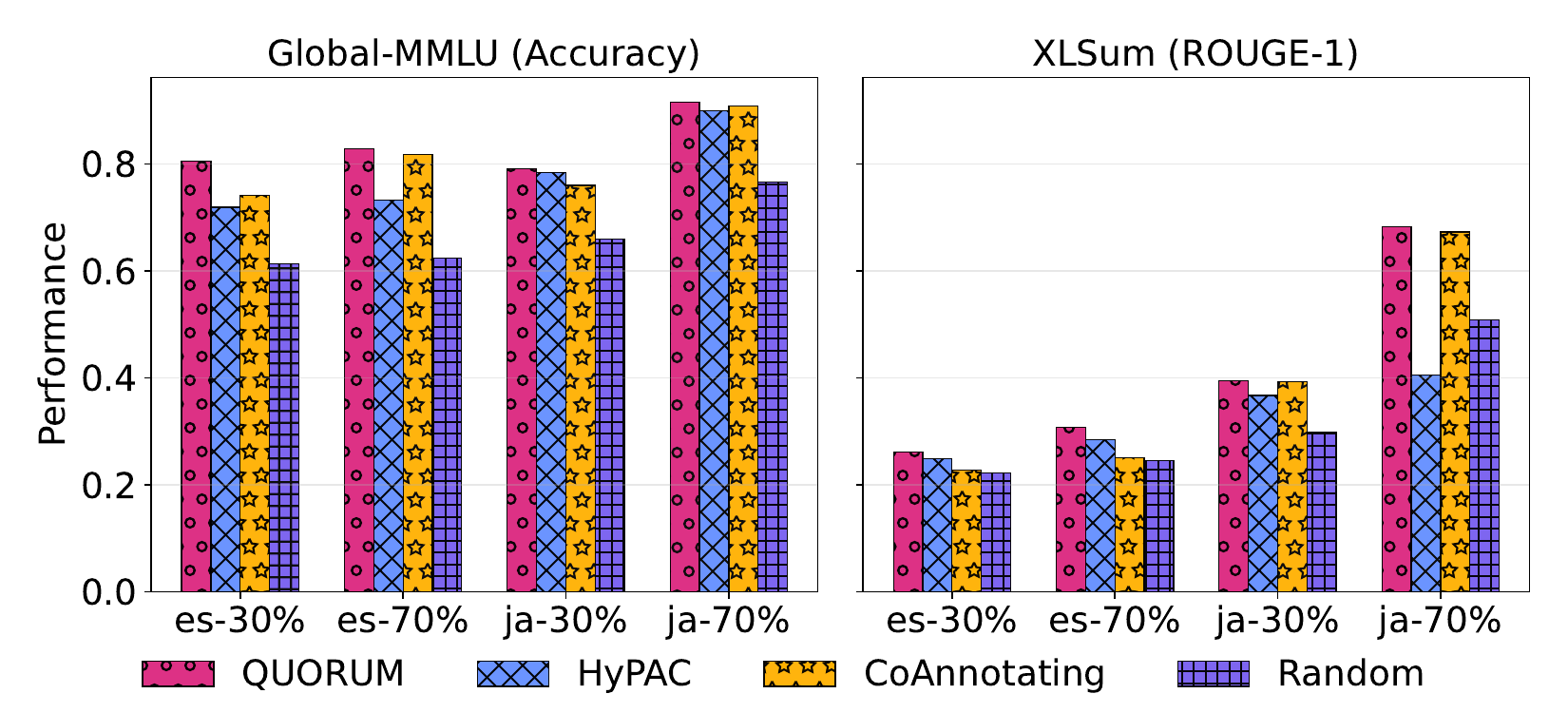}
    \caption{Performance under multilingual settings across Spanish (es) and Japanese (ja) benchmarks for classification (Global-MMLU) and summarization (XLSum) at different annotation budgets (30\% and 70\%).}
    \label{fig:multilingual}
\end{figure}

\subsection*{RQ4. Multilingual Generalization}
To assess whether routing decisions generalize beyond English, we evaluate \quorum on Spanish and Japanese for classification (Global-MMLU) and summarization (XLSum) in the \textsc{dollars} setting. As shown in~\cref{fig:multilingual}, \quorum matches or outperforms competing methods across languages and annotation budgets. Improvements are visible under higher-budget regimes, where additional annotations are allocated more effectively. These gains are achieved while maintaining lower annotation costs, suggesting that the feature-based routing transfers well across languages without requiring language-specific uncertainty estimation. Complete results can be found in~\cref{app:japanese_results}.

\section{Conclusion}

We introduced \quorum, a budget-aware annotation routing framework dynamically allocating samples between human annotators and LLMs. Unlike prior approaches relying on uncertainty, \quorum learns instance difficulty from linguistic and semantic features, enabling cost-efficient and adaptive annotation strategies.
Across classification and summarization benchmarks, our method consistently achieves strong performance–cost trade-offs.

\paragraph{Limitations}
While \quorum demonstrates strong performance across multiple benchmarks, we highlight a few directions for future work. First, our evaluation is conducted primarily on English-language datasets, and extending the framework to additional multilingual or cross-lingual settings would further validate its generality.
Second, although we consider a diverse set of tasks, evaluating \quorum on more specialized domains could provide additional insights into its behavior under different annotation regimes.

\section*{Acknowledgments}
Maria Sofia Bucarelli acknowledges the French government National Research Agency (ANR) through the UCA JEDI (ANR-15-IDEX-01),  through the EUR DS4H (ANR-17-EURE-004), and through the 3IA Cote d'Azur Investments in the project with the reference number ANR-23-IACL-0001. Fabrizio Silvestri and Antonio Purificato acknowledge the project SAP\_RICERCA\_2025\_EVOLVE\_SILVES\_F\_01.

\bibliography{biblio}

\appendix

\section{\quorum Description}
\subsection{Feature Description}
\label{app:feature_description}

Our approach relies on the following features:
\begin{itemize}
    \item Linguistic features~\citep{ziller2026greenserv}: Six numerical values that capture information like sentence complexity, ambiguity, or other text characteristics that make a sample harder to label: 
    \begin{itemize}
        \item Average Word Length: Mean number of characters per word. Longer words often indicate more technical or complex vocabulary. 
        \item Type-Token Ratio: Ratio of unique words to total words. Higher values indicate greater lexical diversity, suggesting more varied or sophisticated language use. 
        \item Average Sentence Length: Mean number of words per sentence. Longer sentences may indicate syntactic complexity or convoluted structure. For languages without explicit word boundaries (\textit{e.g.}, Japanese), word-length–based features were not employed due to the lack of a directly comparable segmentation scheme.
        \item Flesch Reading Ease: Inverted and normalized (0-1) Flesch score~\citep{flesch1948new}, where higher values indicate more difficult text. Combines syllable count and sentence length into a standard readability metric. 
        \item Noun Ratio: Proportion of nouns among POS-tagged content words. Higher noun density may indicate more entity-rich or descriptive text. 
        \item Verb Ratio: Proportion of verbs among POS-tagged content words. Higher verb density may indicate more action-oriented or procedural text.
    \end{itemize}
    \item Embedding features~\citep{hassan2023d}: Each sample is encoded using an embedder (Stella\footnote{\url{https://huggingface.co/NovaSearch/stella_en_1.5B_v5}} for our experiments). For each sample, we compute the cosine distance to its 5 nearest neighbors in the embedding space. This produces a distance vector that serves as a proxy for sample atypicality or semantic isolation .
\end{itemize}

\subsection{Proof of Proposition~\ref{prop:posterior_concentration}}
\label{app:proof_posterior}
For each annotator \(a\), rewards follow:
\begin{equation*}
r_{a,t}= x_t^\top \theta_a^* + \eta_t
\end{equation*}
With Gaussian noise $\eta_t\sim \mathcal N(0,\sigma^2)$.
Bayesian linear regression updates yield posterior parameters:
\begin{align*}
\Sigma_t^{-1} &=\Sigma_0^{-1}+ \frac1{\sigma^2} \sum_{s<t} x_sx_s^\top \\
\mu_t &= \Sigma_t \left(\Sigma_0^{-1}\mu_0 + \frac1{\sigma^2} \sum_{s<t} r_sx_s \right).
\end{align*}

Applying the concentration bounds for Bayesian linear
regression proved in ~\citet{abbasi2011improved} (precisely Theorem 2 in \citet{abbasi2011improved}), with probability
\(1-\delta\):
\begin{equation*}
\|\mu_t-\theta^*\|_{{\Sigma_t}^{-1}} \le \beta_t .
\end{equation*}
Now:
\begin{align*}
& \left|x_t^\top(\mu_{a,t}-\theta_a^*) \right|
= \left|x_t^\top \Sigma_{a,t}^{1/2} \Sigma_{a,t}^{- 1/2}\top(\mu_{a,t}-\theta_a^*)\right| \\ 
& = \left|
\left(\Sigma_{a,t}^{1/2}x_t\right)^\top
\left(\Sigma_{a,t}^{-1/2}(\mu_{a,t}-\theta_a^*)\right)
\right| 
\end{align*}
Using Cauchy-Schwarz we have: 
\begin{align*}
\left|x_t^\top(\mu_{a,t}-\theta_a^*) \right| \leq \left\|\Sigma_{a,t}^{1/2}x_t\right\|_2 
\left\|\mu_{a,t}-\theta_a^*\right\|_{\Sigma_{a,t}^{-1}}
\end{align*}
Namely:
\begin{equation*}
|x_t^\top (\mu_t-\theta^*)|\le \beta_t \sqrt{x_t^\top\Sigma_t x_t}
\end{equation*}

If the annotator is sampled infinitely often on sufficiently diverse contexts,
so that: \[ \lambda_{\min}(\Sigma_t^{-1})\to\infty\] Then for bounded contexts
\(\|x_t\|_2\le L\):
\[
x_t^\top\Sigma_t x_t
\le
\frac{L^2}{\lambda_{\min}(\Sigma_t^{-1})}
\to0
\]

Thus posterior uncertainty shrinks asymptotically and quality estimates converge toward the true annotator parameters. $\square$

\subsection{Proof of Proposition~\ref{prop:regret}}
\label{app:proof_regret}

\newtheorem*{propositionmanual}{Proposition~\ref{prop:regret} (formal)}

\begin{propositionmanual}
Assume the linear contextual reward model:
\[
r_{a,t}=x_t^\top\theta_a^*+\eta_{a,t}
\]
with bounded contexts \(\|x_t\|_2\le L\), bounded parameters
\(\|\theta_a^*\|_2\le S\), and conditionally sub-Gaussian noise. For each
context \(x_t\), define the oracle annotator:
\[
a_t^*
\in
\arg\max_{a\in\mathcal A} x_t^\top\theta_a^*
\]
Assume that the oracle annotator is uniformly separated from the second-best
annotator, i.e.:
\[
\Delta_{\min}
=
\inf_{t\ge1}
\left(
x_t^\top\theta_{a_t^*}^*
-
\max_{a\neq a_t^*}x_t^\top\theta_a^*
\right)
>0
\]
Consider the core routing rule that explores with probability:
\[
\epsilon_t=\frac{\epsilon_0}{1+\gamma t},
\]
and otherwise exploits according to:
\[
\hat a_t
\in
\arg\max_{a\in\mathcal A} x_t^\top\mu_{a,t}
\]
Assume that the exploration mechanism samples each feasible annotator infinitely
often on sufficiently diverse contexts, so that, for every annotator \(a\):
\[
\frac{\lambda_{\min}(\Sigma_{a,t}^{-1})}{\beta_{a,t}^2}
\to\infty
\]
Then, on the posterior concentration event of Proposition~\ref{prop:posterior_concentration},
the average regret vanishes:
\[
\lim_{T\to\infty}\frac{R(T)}{T}=0
\]
Where:
\[
R(T)
=
\sum_{t=1}^T
\left(
x_t^\top\theta_{a_t^*}^*
-
x_t^\top\theta_{\hat a_t}^*
\right)
\]
\end{propositionmanual}
\begin{proof}
Let:
\[
\Delta_t
=
x_t^\top\theta_{a_t^*}^*
-
x_t^\top\theta_{\hat a_t}^*
\]
denote the instantaneous regret.

We work on the high-probability event from
Proposition~\ref{prop:posterior_concentration}, on which, for every annotator
\(a\):
\[
\left|x_t^\top(\mu_{a,t}-\theta_a^*)\right|
\le
\beta_{a,t}\sqrt{x_t^\top\Sigma_{a,t}x_t}
\]

First, we show that the prediction error vanishes. Since \(\|x_t\|_2\le L\):
\[
x_t^\top\Sigma_{a,t}x_t
\le
\lambda_{\max}(\Sigma_{a,t})\|x_t\|_2^2
\]
Moreover:
\[
\lambda_{\max}(\Sigma_{a,t})
=
\frac{1}{\lambda_{\min}(\Sigma_{a,t}^{-1})}
\]
Therefore:
\[
x_t^\top\Sigma_{a,t}x_t
\le
\frac{L^2}{\lambda_{\min}(\Sigma_{a,t}^{-1})}
\]
Substituting this into the concentration bound gives:
\[
\left|x_t^\top(\mu_{a,t}-\theta_a^*)\right|
\le
\frac{\beta_{a,t}L}
{\sqrt{\lambda_{\min}(\Sigma_{a,t}^{-1})}}
\]
By assumption:
\[
\frac{\lambda_{\min}(\Sigma_{a,t}^{-1})}{\beta_{a,t}^2}
\to\infty
\]
And hence:
\[
\frac{\beta_{a,t}L}
{\sqrt{\lambda_{\min}(\Sigma_{a,t}^{-1})}}
\to0
\]
Thus, for every annotator \(a\):
\[
\left|x_t^\top(\mu_{a,t}-\theta_a^*)\right|
\to0
\]
Since \(\mathcal A\) is finite, this convergence holds uniformly over
annotators:
\[
\max_{a\in\mathcal A}
\left|x_t^\top(\mu_{a,t}-\theta_a^*)\right|
\to0
\]

Now use the margin assumption. Since the maximum prediction error over
annotators goes to zero, there exists a finite time \(T_0\) such that, for all
\(t\ge T_0\) and all \(a\in\mathcal A\):
\[
\left|x_t^\top(\mu_{a,t}-\theta_a^*)\right|
\le
\frac{\Delta_{\min}}{4}.
\]
Fix any \(t\ge T_0\), and suppose the algorithm is in an exploitation round.
We show that the greedy rule selects the oracle annotator.

For the oracle annotator \(a_t^*\), we have:
\[
x_t^\top\mu_{a_t^*,t}
\ge
x_t^\top\theta_{a_t^*}^*
-
\frac{\Delta_{\min}}{4}
\]
For any non-oracle annotator \(a\neq a_t^*\), the margin assumption gives:
\[
x_t^\top\theta_a^*
\le
x_t^\top\theta_{a_t^*}^*
-
\Delta_{\min}
\]
Using again the prediction-error bound:
\[
x_t^\top\mu_{a,t}
\le
x_t^\top\theta_a^*
+
\frac{\Delta_{\min}}{4}
\]
Combining the last two inequalities:
\[
x_t^\top\mu_{a,t}
\le
x_t^\top\theta_{a_t^*}^*
-
\Delta_{\min}
+
\frac{\Delta_{\min}}{4}
=
x_t^\top\theta_{a_t^*}^*
-
\frac{3\Delta_{\min}}{4}
\]
On the other hand:
\[
x_t^\top\mu_{a_t^*,t}
\ge
x_t^\top\theta_{a_t^*}^*
-
\frac{\Delta_{\min}}{4}
\]
Therefore, for every \(a\neq a_t^*\):
\[
x_t^\top\mu_{a_t^*,t}
>
x_t^\top\mu_{a,t}
\]
Hence the exploitation rule:
\[
\hat a_t\in\arg\max_{a\in\mathcal A}x_t^\top\mu_{a,t}
\]
selects \(\hat a_t=a_t^*\). Thus every exploitation round after \(T_0\) has
zero regret.

For the exploration rounds, let \(Z_t\in\{0,1\}\) be the indicator that round
\(t\) is an exploration round. By construction of the algorithm:
\[
\mathbb P(Z_t=1\mid\mathcal F_{t-1})=\epsilon_t
\]
where \(\mathcal F_{t-1}\) denotes the history before round \(t\). Since
\(Z_t\) is binary, this implies:
\[
\mathbb E[Z_t\mid\mathcal F_{t-1}]=\epsilon_t
\]
Defining:
\[
Y_t=Z_t-\epsilon_t
\]
Then:
\[
\mathbb E[Y_t\mid\mathcal F_{t-1}]
=
\mathbb E[Z_t-\epsilon_t\mid\mathcal F_{t-1}]
=
\epsilon_t-\epsilon_t
=
0
\]
Defining $M_T$ as:
\[
M_T=\sum_{t=1}^T Y_t
=
\sum_{t=1}^T(Z_t-\epsilon_t)
\]
We can show it is a martingale with bounded increments.
Since:
\[
\epsilon_t=\frac{\epsilon_0}{1+\gamma t}
\]
We have:
\[
\sum_{t=1}^T\epsilon_t
\le
\frac{\epsilon_0}{\gamma}\log(1+\gamma T)
\]
Therefore:
\[
\frac{1}{T}\sum_{t=1}^T\epsilon_t\to0
\]
Since \(Y_t=Z_t-\epsilon_t\) is a martingale difference sequence and
\(|Y_t|\le1\), we have:
\[
\sum_{t=1}^{\infty}\frac{\mathbb E[Y_t^2]}{t^2}
\le
\sum_{t=1}^{\infty}\frac1{t^2}<\infty
\]
By the strong law for martingale differences
\citep[Theorem~2.18]{hall2014martingale}, it follows that
\[
\frac1T\sum_{t=1}^T(Z_t-\epsilon_t)\to0
\quad\text{almost surely.}
\]
Consequently:
\[
\frac1T\sum_{t=1}^T Z_t
=
\frac1T\sum_{t=1}^T\epsilon_t
+
\frac1T\sum_{t=1}^T(Z_t-\epsilon_t)
\to0
\]
Thus exploration rounds form a vanishing fraction of all rounds.

Finally, instantaneous regret is bounded. Indeed, since \(\|x_t\|_2\le L\) and
\(\|\theta_a^*\|_2\le S\):
\[
\left|x_t^\top\theta_a^*-x_t^\top\theta_b^*\right|
\le
|x_t^\top\theta_a^*|+|x_t^\top\theta_b^*|
\le
2LS
\]
for any \(a,b\). Hence \(\Delta_t\le2LS\).

After time \(T_0\), exploitation rounds have zero regret, and only exploration
rounds can contribute regret. Therefore:
\[
R(T)
\le
2LS\,T_0
+
2LS\sum_{t=T_0}^T Z_t
\]
Dividing by \(T\), we obtain:
\[
\frac{R(T)}{T}
\le
\frac{2LS\,T_0}{T}
+
2LS\frac1T\sum_{t=T_0}^T Z_t
\]
The first term goes to zero because \(T_0\) is finite. The second term goes to
zero because exploration rounds have vanishing frequency. Therefore:
\[
\frac{R(T)}{T}\to0
\]
\end{proof}

\begin{algorithm}[t]
\caption{QUORUM --- \textsc{dollars}}
\label{pseudocode:dollars}
\begin{algorithmic}[1]
\Require LLM predictions $\mathcal{P} = \{P_1, \ldots, P_K\}$, human labels $\mathbf{y}$, budget $B$
\Ensure Final outputs $\mathbf{o} = [o_1, \ldots, o_T]$
\State Initialize posteriors: $\boldsymbol{\mu}_a \gets \mathbf{0}$, $\boldsymbol{\Sigma}_a \gets \mathbf{I}$ for LLM arms
\Statex \textbf{Phase 1: Calibration}
\State $N_{\text{cal}} \gets \lfloor 0.20 \cdot \lfloor B / c_{\text{human}} \rfloor \rfloor$
\State $\mathcal{C} \gets$ top-$N_{\text{cal}}$ samples by Mahalanobis uncertainty
\For{$t \in \mathcal{C}$}
    \State Query human for $y_t$; for each LLM arm $a$: update posterior with reward $r$
\EndFor
\Statex \textbf{Phase 2: Routing}
\For{each uncovered sample $t \notin \mathcal{C}$}
    \State $a^* \gets \Call{Select}{\mathbf{x}_t, \text{available arms}, t}$
    \State Record annotation from $a^*$; update budget
\EndFor
\Statex \textbf{Phase 3: Extra Annotations} 
\While{budget remains}
    \State Pick highest-priority sample by $P_t$
    \State Select arm; record annotation
\EndWhile
\Statex \textbf{Phase 4: Aggregation}
\For{$t = 1, \ldots, T$}
    \State $o_t \gets$ weighted majority vote
\EndFor
\end{algorithmic}
\end{algorithm}

\begin{algorithm}[t]
\caption{QUORUM --- \textsc{Auditor Style}}
\label{pseudocode:auditor}
\begin{algorithmic}[1]
\Require LLM predictions $\mathcal{P}$, human labels $\mathbf{y}$, human budget $B$
\Ensure Final outputs $\mathbf{o}$
\State $\mathbf{p} \gets$ predictions from cheapest LLM
\State $\mathcal{C} \gets$ random $0.2B$ calibration samples
\State $\mathcal{E} \gets \{t \in \mathcal{C} : p_t \neq y_t\}$
\If{$|\mathcal{E}| \geq 2$}
    \State Sort remaining samples by distance to error centroid (ascending)
\EndIf
\State Assign human labels to $\mathcal{C} \cup$ top-$(B - |\mathcal{C}|)$ nearest samples; use LLM elsewhere
\end{algorithmic}
\end{algorithm}

\subsection{\approach Pseudocode}
\label{pseudocode_sec}

\Cref{pseudocode:dollars} summarizes \quorum under the \textsc{Dollars} setting, where routing decisions are constrained by a monetary budget and annotators have heterogeneous costs. The procedure consists of an initial calibration phase, adaptive routing, optional budget-aware re-annotation of difficult samples, and task-specific aggregation. \verb|SELECT| refers to the Routing Mechanism paragraph of~\cref{routing_human_llm}.
\Cref{pseudocode:auditor} summarizes the \textsc{Auditor Style} setting, where a fixed human annotation budget is available and the router decides whether to allocate samples to humans or to the cheapest LLM annotator. Human supervision is concentrated on calibration examples and samples estimated to be close to previously observed model errors.



\section{Experiments - Additional Details}
\label{app:details}

\paragraph{Experimental Setup}
All experiments are conducted on a machine equipped with 96 CPU cores (AMD EPYC 7R32), 4 NVIDIA A10G GPUs, and 373 GB of system memory.

\paragraph{Prompts}
\label{app:prompts}
The prompt we use to obtain LLM annotations for classification is the following:

\begin{tcolorbox}[colback=gray!10, colframe=gray!50]
\small{
{\ttfamily This is a text classification task. The possible labels are [...] while the indices of the labels are [...].
Output only 'Label: the predicted index of the label that you predict'. This is the sentence to classify: [...]. 
}
}
\end{tcolorbox}

The prompt we use to obtain LLM annotations for summarization is the following:

\begin{tcolorbox}[colback=gray!10, colframe=gray!50]
\small{
{\ttfamily This is a text summarization task. Output only the summarized text and nothing more. Be concise and summarize the following text [...]
}
}
\end{tcolorbox}

The prompt we use to obtain LLM annotations for Q\&A is the following:

\begin{tcolorbox}[colback=gray!10, colframe=gray!50]
\small{
{\ttfamily This is a question-answering task. Select the right answer from the available choices. Output only 'Label: the predicted index of the label that you predict'. This is the question [...] and the possible answers are: [...]
}
}
\end{tcolorbox}

The prompt we use to obtain the confidence is:

\begin{tcolorbox}[colback=gray!10, colframe=gray!50]
\small{
{\ttfamily How likely is it that the following text is correctly summarized/classified? Output the probability (a number between 0 and 1) Text: [...] Prediction: [...]
}
}
\end{tcolorbox}

We set \verb|max_tokens=1000| and \verb|temperature=1|.

\paragraph{Datasets} We evaluate on a diverse set of benchmarks spanning several tasks like classification with AG's News, SST2~\citep{socher-etal-2013-recursive}, IMDB~\citep{maas-EtAl:2011:ACL-HLT2011} and PubMed~\citep{dernoncourt2017pubmed}, multiple-choice QA with Global-MMLU~\citep{singh2025global} and MMLU-Redux~\cite{gema2025we} and abstractive summarization with CNN/DailyMail~\cite{see-etal-2017-get} and XLSum~\cite{hasan-etal-2021-xl}. This selection covers varying difficulty levels, class distributions, and task formats.

\begin{table}[t]
  \centering
  \caption{Number of classes, class distribution and number of samples of the selected datasets. NA indicates summarization datasets with no classes.}
    \label{tab:datasets} 
  \resizebox{\linewidth}{!}{
  \begin{tabular}{lccc}
        \toprule
        \textbf{Name} & \#Classes & Class Distribution (\%)  & \#Samples  \\
        \midrule
        \href{https://huggingface.co/datasets/fancyzhx/ag_news}{AG's News}                       & 4         & [25,25,25,25]   & 7,600    \\
      \href{https://huggingface.co/datasets/stanfordnlp/sst2}{SST2} & 2 & [49.9,51.1] & 870 \\
      \href{https://huggingface.co/datasets/CohereLabs/Global-MMLU}{Global-MMLU}$^*$ & 4 & [23, 25, 25, 27] &  14,300  \\
         \href{https://huggingface.co/datasets/edinburgh-dawg/mmlu-redux}{MMLU-Redux}$^*$ & 4 & [22, 25, 25, 28] &  3,000  \\
         \href{https://huggingface.co/datasets/stanfordnlp/imdb}{IMDB}$^*$   & 2         & [50,50] & 25,000     \\
         \href{https://huggingface.co/datasets/pietrolesci/pubmed-20k-rct}{PubMed}$^*$ & 5 & [10.4,15.45, 33.42,7.89,32.84]   & 30,000 \\
        \href{https://huggingface.co/datasets/rbnuria/SentiMP-En}{SentiMP}$^*$ & 3 & [28,21,51]  & 480  \\\href{https://huggingface.co/datasets/abisee/cnn_dailymail}{CNN}$^*$ & NA & NA & 11,490 \\
         \href{https://huggingface.co/datasets/csebuetnlp/xlsum}{XLSum}$^*$ & NA & NA  & 4,763  \\

        \bottomrule
    \end{tabular}
    }
\end{table}

\section{Additional Results}

\subsection{Feature Analysis of \quorum}
\label{app:feature_routing}
\begin{figure*}[ht]
    \centering
    \includegraphics[width=\linewidth]{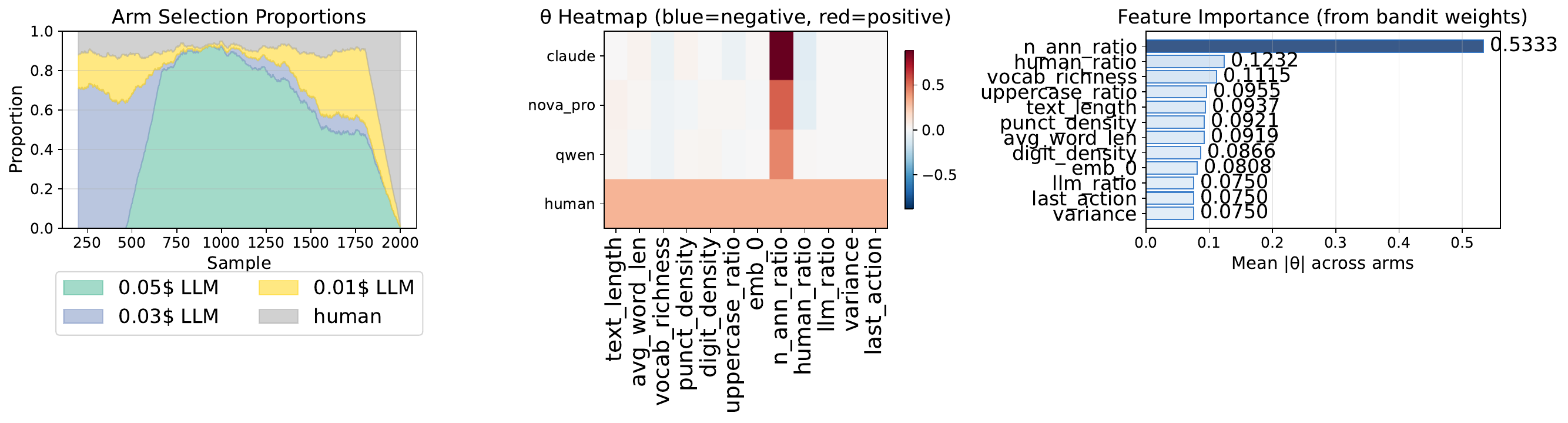}
    \caption{Analysis of the learned bandit routing policy on the AG's News dataset (window of 2000 samples) in the \textsc{dollars} setting. \textbf{Left:} arm selection proportions over time, showing convergence from exploration to exploitation. \textbf{Center:} heatmap of posterior weight vectors across features and arms (blue=negative, red=positive). \textbf{Right:} feature importance computed as mean absolute weight across arms.} \label{fig:bandit_dashboard}
\end{figure*}

\begin{table*}[ht]
\centering
\caption{Sensitivity analysis of \quorum with respect to key routing and aggregation hyperparameters. We report downstream accuracy while varying the human escalation threshold ($\tau_H$), human voting weight ($w_H$), initial exploration rate ($\epsilon_0$), exploration decay factor ($\gamma$) and calibration size.}
\label{tab:sensitivity}
\begin{minipage}[t]{0.28\textwidth}
\centering
\begin{tabular}{cc}
\toprule
$\mathbf{\tau_{H}}$ Value & Accuracy \\ 
\midrule
0.1 & 0.780 \\
0.3 & 0.781 \\
0.5 & \textbf{0.879} \\
0.7 & 0.875 \\
\bottomrule
\end{tabular}
\subcaption{$\tau_H$ sensitivity.}
\end{minipage}
\hfill
\begin{minipage}[t]{0.28\textwidth}
\centering
\begin{tabular}{cc}
\toprule
$\mathbf{w_{H}}$ Value & Accuracy \\ 
\midrule
1 & 0.790 \\
5 & 0.831 \\
10 & \textbf{0.879} \\
15 & 0.871 \\
\bottomrule
\end{tabular}
\subcaption{$w_H$ sensitivity.}
\end{minipage}
\hfill
\begin{minipage}[t]{0.28\textwidth}
\centering
\begin{tabular}{cc}
\toprule
$\mathbf{\epsilon_0}$ Value & Accuracy \\ 
\midrule
0.1 & 0.861 \\
0.3 & \textbf{0.879} \\
0.5 & 0.873 \\
0.7 & 0.865 \\
\bottomrule
\end{tabular}
\subcaption{$\epsilon_0$ sensitivity.}
\end{minipage}
\\
\begin{minipage}[t]{0.28\textwidth}
\centering
\begin{tabular}{cc}
\toprule
$\mathbf{\gamma}$ Value & Accuracy \\ 
\midrule
0.001 & 0.875 \\
0.005 & \textbf{0.879} \\
0.01 & 0.842 \\
0.05 & 0.857 \\
\bottomrule
\end{tabular}
\subcaption{$\gamma$ sensitivity.}
\end{minipage}
\begin{minipage}[t]{0.28\textwidth}
\centering
\begin{tabular}{cc}
\toprule
Calibration Size & Accuracy \\ 
\midrule
0.2 & \textbf{0.879} \\
0.4 & 0.874 \\
0.6 & 0.871 \\
0.8 & 0.873 \\
\bottomrule
\end{tabular}
\subcaption{Calibration size sensitivity.}
\end{minipage}

\end{table*}

To understand how \quorum allocates annotations, we inspect the learned bandit policy in~\cref{fig:bandit_dashboard}. The arm selection proportions image (left) show that the router initially explores all annotators uniformly, then progressively converges toward lower-cost annotators for routine instances, while reserving higher-cost annotators (both human and LLM) for instances where predicted difficulty is highest.

The feature importance analysis (right) shows that $b_t$, the remaining budget fraction, is by far the dominant routing signal, followed by contextual features derived from $m_t$ (the annotation mask) and linguistic components of $\phi(d_t)$ such as vocabulary richness. This confirms that routing decisions are primarily driven by budget state and annotation history rather than embedding-based features alone. In particular, the model learns to regulate both the overall annotation ratio and the human annotation ratio as explicit proxies for budget consumption and allocation pressure, allowing it to avoid premature depletion of high-cost annotations while still maintaining coverage on uncertain instances. The heatmap (center) further reveals arm-specific patterns: the human arm exhibits strong positive weights on budget and mask features, confirming that escalation to $H$ is triggered when prior annotations are insufficient or conflicting.

\subsection{Sensitivity analysis of \quorum}
\label{app:sensitivity}

To assess robustness to hyperparameter choices, we evaluate \quorum under different values of the main routing and aggregation parameters (\cref{tab:sensitivity}). Sensitivity experiments are performed on the AG's News dataset under the \textsc{dollars} setting with a total budget of \$450. Overall, performance remains relatively stable across configurations, suggesting that the framework is not highly sensitive to precise tuning.
The human escalation threshold $\tau_H$ has the largest impact: moderate values ($\tau_H=0.5$) achieve the best performance, whereas overly aggressive or overly conservative escalation leads to lower accuracy. This indicates that balancing reliance on LLMs and human supervision is critical.
Similarly, increasing the human voting weight $w_H$ improves performance up to an intermediate regime ($w_H=10$), confirming that prioritizing human annotations is beneficial, but excessive weighting yields diminishing returns.
Exploration parameters exhibit comparable behaviour. Moderate initial exploration ($\epsilon_0=0.3$) and slow decay ($\gamma=0.005$) provide the strongest results, whereas excessive exploration or premature exploitation degrade performance. These findings support the design choice of combining adaptive exploration with budget-aware routing.
Overall, the results suggest that \quorum is robust to moderate hyperparameter variation and does not require extensive tuning to obtain competitive performance. The calibration-size analysis suggests that allocating a relatively small portion of the budget to the initial calibration stage is sufficient to obtain reliable annotator estimates. Performance peaks when 20\% of the affordable budget is reserved for calibration, while larger calibration fractions gradually reduce accuracy. This indicates that excessive upfront calibration may limit the resources available for adaptive routing and iterative annotation refinement.

\subsection{Time Analysis of \quorum}
\label{app:time_analysis}
\Cref{tab:time} reports the routing overhead required by different routing strategies before annotation decisions can be made. Existing approaches relying on confidence estimation (CT) require an additional LLM call for every sample to obtain uncertainty signals, resulting in substantial computational costs. On datasets of roughly 2,000 instances, this process can require approximately two hours. Training-based methods (TT), such as ARAIDA and SANT, are also expensive due to supervised fitting procedures, typically requiring around 20 minutes for the same dataset size.

In contrast, \quorum only performs embedding extraction (ET) and lightweight feature computation, reducing preprocessing to approximately four minutes on datasets of similar size. Across benchmarks, \quorum consistently achieves lower computational overhead than confidence- or training-based alternatives while maintaining competitive annotation performance. These results suggest that replacing uncertainty estimation with feature-based difficulty signals improves not only routing effectiveness but also practical scalability.

\begin{table}[t]
    \centering
    \caption{Routing overhead (seconds) required before routing decisions across datasets ($\downarrow$). ET denotes embedding extraction time, CT confidence estimation time via additional LLM prompting, and TT model training time.}
    \label{tab:time}
    \resizebox{\linewidth}{!}{%
    \begin{tabular}{cccc}
    \toprule
    Method & AG's News & PubMed & Global-MMLU \\
    \midrule
    Random & 0.007 & 0.035 & 0.014 \\
    ARAIDA & 4.988+TT & 19.879+TT & 12.250+TT \\
    CDI & 0.503+CT & 0.547+CT & 0.661+CT \\
    CoAnnotating & 0.006+CT & 0.038+CT & 0.014+CT \\
    HyPAC & 0.046+CT & 0.150+CT & 0.075+CT \\
    SANT & 18.052+TT & 161.014+TT & 40.853+TT \\
    \approach & 0.494+ET & 1.467+ET & 0.593+ET \\
    \bottomrule
    \end{tabular}
    }
\end{table}

\subsection{Additional Multilingual Results}
\label{app:japanese_results}
To assess the generalizability of our approach beyond English, we evaluate all methods on Japanese and Spanish language tasks using Global-MMLU (ja, es) for multiple-choice Question Answering and XLSum (ja, es) for abstractive summarization. These datasets allow us to test whether the annotation strategies remain effective when applied to a typologically distant language with different linguistic properties.~\Cref{tab:japanese_auditor,tab:japanesedollars} report results under the \textsc{auditor style} and \textsc{dollars} settings, respectively. Our method consistently achieves the best or second-best performance across both tasks and budget levels, demonstrating that its effectiveness is not limited to English and transfers well to multilingual scenarios. Notably, the performance gap with competing methods widens at higher budgets: at $70\%$ human annotation, \quorum reaches $0.946$ accuracy on Global-MMLU (ja), outperforming the second-best method by nearly $4$ percentage points. These results confirm that our method generalizes effectively to typologically distant languages, even when LLM annotators are less reliable.

\begin{table*}[!t]
\centering
\caption{Performance in the \textsc{auditor style} setting under different human annotation budgets. We report Accuracy or ROUGE-1 depending on the task ($\uparrow$), Machine Cumulative Accuracy ($\uparrow$), and total annotation cost ($\downarrow$) for varying percentages of human-labeled data. \textbf{Bold} indicates the best-performing method, while \underline{underlined} the second best. Value in parentheses represents the average performance obtained by LLM annotators on that dataset.}\label{tab:japanese_auditor}
\resizebox{\linewidth}{!}{%
\begin{tabular}{llccc|ccc|ccc|ccc}
\toprule
\multirow{2}{*}{\textbf{Dataset}} & \multirow{2}{*}{\textbf{Method}} & \multicolumn{3}{c}{\textbf{10\%}} & \multicolumn{3}{c}{\textbf{30\%}} & \multicolumn{3}{c}{\textbf{50\%}} & \multicolumn{3}{c}{\textbf{70\%}} \\
\cmidrule(lr){3-5} \cmidrule(lr){6-8} \cmidrule(lr){9-11} \cmidrule(lr){12-14}
& & Accuracy/ROUGE & MCA & Cost & Accuracy/ROUGE & MCA & Cost & Accuracy/ROUGE & MCA & Cost & Accuracy/ROUGE & MCA & Cost \\
\midrule
\multirow{7}{*}{\shortstack{Global-MMLU (ja) \\ (LLM: 0.640) \\ 14,040 samples}}
& Random & $0.641 \pm 0.001$ & $0.305 \pm 0.013$ & $\underline{267}$ & $0.729 \pm 0.002$ & $0.313 \pm 0.007$ & $\underline{519}$ & $0.729 \pm 0.002$ & $0.165 \pm 0.003$ & $\underline{772}$ & $0.796 \pm 0.002$ & $0.243 \pm 0.002$ & $\underline{1025}$ \\
& ARAIDA & $0.679 \pm 0.014$ & $0.383 \pm 0.109$ & $281$ & $0.732 \pm 0.018$ & $0.312 \pm 0.012$ & $562$ & $0.759 \pm 0.052$ & $0.315 \pm 0.015$ & $842$ & $0.782 \pm 0.082$ & $0.318 \pm 0.018$ & $1123$ \\
& CDI & $0.691 \pm 0.001$ & $0.355 \pm 0.014$ & $281$ & $0.751 \pm 0.002$ & $0.318 \pm 0.006$ & $562$ & $0.823 \pm 0.002$ & $\underline{0.335 \pm 0.004}$ & $842$ & $0.893 \pm 0.002$ & $0.338 \pm 0.002$ & $1123$ \\
& CoAnnotating & $0.692 \pm 0.001$ & $0.358 \pm 0.001$ & $281$ & $\underline{0.760 \pm 0.001}$ & $\underline{0.348 \pm 0.001}$ & $562$ & $0.818 \pm 0.001$ & $0.323 \pm 0.001$ & $842$ & $\underline{0.909 \pm 0.001}$ & $\underline{0.361 \pm 0.001}$ & $1123$ \\
& HyPAC & $\underline{0.715 \pm 0.001}$ & $0.349 \pm 0.009$ & $281$ & $\mathbf{0.784 \pm 0.001}$ & $\mathbf{0.350 \pm 0.002}$ & $562$ & $\mathbf{0.838 \pm 0.001}$ & $0.331 \pm 0.002$ & $842$ & $0.899 \pm 0.010$ & $0.342 \pm 0.008$ & $1123$ \\
& SANT & $0.689 \pm 0.011$ & $\underline{0.389 \pm 0.072}$ & $281$ & $0.738 \pm 0.019$ & $0.307 \pm 0.044$ & $562$ & $0.761 \pm 0.038$ & $0.335 \pm 0.058$ & $842$ & $0.780 \pm 0.064$ & $0.293 \pm 0.010$ & $1123$ \\
& \approach & $\mathbf{0.721 \pm 0.001}$ & $\mathbf{0.394 \pm 0.008}$ & $\mathbf{267}$ & $0.748 \pm 0.001$ & $0.335 \pm 0.005$ & $\mathbf{519}$ & $\mathbf{0.857 \pm 0.001}$ & $\mathbf{0.364 \pm 0.003}$ & $\mathbf{772}$ & $\mathbf{0.946 \pm 0.001}$ & $\mathbf{0.400 \pm 0.002}$ & $\mathbf{1025}$ \\ \midrule
\multirow{7}{*}{\shortstack{Global-MMLU (es) \\ (LLM: 0.676) \\ 14,300 samples}}
& Random & $0.710 \pm 0.001$ & $0.289 \pm 0.011$ & $267$ & $0.788 \pm 0.001$ & $0.289 \pm 0.004$ & $519$ & $0.826 \pm 0.001$ & $0.260 \pm 0.002$ & $772$ & $0.853 \pm 0.001$ & $0.229 \pm 0.002$ & $1025$ \\
& ARAIDA & $0.727 \pm 0.013$ & $\underline{0.358 \pm 0.112}$ & $281$ & $0.777 \pm 0.013$ & $0.275 \pm 0.004$ & $562$ & $0.800 \pm 0.042$ & $0.278 \pm 0.004$ & $842$ & $0.822 \pm 0.069$ & $0.280 \pm 0.005$ & $1123$ \\
& CDI & $0.740 \pm 0.001$ & $0.295 \pm 0.009$ & $281$ & $0.795 \pm 0.001$ & $0.281 \pm 0.003$ & $562$ & $\underline{0.854 \pm 0.002}$ & $0.286 \pm 0.004$ & $842$ & $0.912 \pm 0.001$ & $\underline{0.288 \pm 0.001}$ & $1123$ \\
& CoAnnotating & $0.741 \pm 0.001$ & $0.303 \pm 0.001$ & $281$ & $\mathbf{0.802 \pm 0.001}$ & $\underline{0.303 \pm 0.001}$ & $562$ & $0.852 \pm 0.001$ & $0.283 \pm 0.001$ & $842$ & $0.910 \pm 0.001$ & $0.284 \pm 0.001$ & $1123$ \\
& HyPAC & $\underline{0.742 \pm 0.001}$ & $0.313 \pm 0.007$ & $281$ & $\underline{0.802 \pm 0.001}$ & $\mathbf{0.305 \pm 0.003}$ & $562$ & $0.855 \pm 0.001$ & $0.289 \pm 0.003$ & $842$ & $\mathbf{0.915 \pm 0.001}$ & $0.292 \pm 0.001$ & $1123$ \\
& SANT & $0.738 \pm 0.009$ & $0.324 \pm 0.072$ & $281$ & $0.798 \pm 0.011$ & $0.298 \pm 0.008$ & $562$ & $0.851 \pm 0.035$ & $\mathbf{0.310 \pm 0.013}$ & $842$ & $0.900 \pm 0.058$ & $\mathbf{0.308 \pm 0.011}$ & $1123$ \\
& \approach & $\mathbf{0.749 \pm 0.002}$ & $\mathbf{0.379 \pm 0.016}$ & $\mathbf{267}$ & $0.798 \pm 0.002$ & $0.290 \pm 0.006$ & $\mathbf{519}$ & $\mathbf{0.856 \pm 0.002}$ & $\underline{0.290 \pm 0.004}$ & $\mathbf{772}$ & $\underline{0.912 \pm 0.001}$ & $0.287 \pm 0.001$ & $\mathbf{1025}$ \\
\midrule
& & ROUGE-1 & MCA & Cost & ROUGE-1 & MCA & Cost & ROUGE-1 & MCA & Cost & ROUGE-1 & MCA & Cost \\ \midrule
\multirow{5}{*}{\shortstack{XLSum (ja) \\ (LLM: 0.166) \\ 880 samples}}
& Random & $0.243 \pm 0.003$ & - & \underline{17} & $0.388 \pm 0.006$ & - & \underline{32} & $\underline{0.532 \pm 0.005}$ & - & \underline{48}& $0.678 \pm 0.001$ & - & \underline{64} \\
& CDI & $0.247 \pm 0.001$ & - & 18 & $0.381 \pm 0.004$ & - & 34 & $0.529 \pm 0.004$ & - & 53 & $0.673 \pm 0.003$ & - & 71 \\
& CoAnnotating & $0.250 \pm 0.001$ & - & 18 & $\underline{0.395 \pm 0.001}$ & - & 34 & $0.541 \pm 0.001$ & - & 53 & $\underline{0.682 \pm 0.001}$ & - & 71 \\
& HyPAC & $\underline{0.255 \pm 0.004}$ & - & 18 & $0.367 \pm 0.003$ & - & 34 & $0.385 \pm 0.001$ & - & 53  & $0.405 \pm 0.003$ & - & 71 \\
& \approach & $\mathbf{0.258 \pm 0.002}$ & - & \textbf{17}  & $\mathbf{0.403 \pm 0.001}$ & - & \textbf{32} & $\mathbf{0.543 \pm 0.002}$ & - & \textbf{48} & $\mathbf{0.683 \pm 0.002}$ & - & \textbf{64} \\ \midrule
\multirow{5}{*}{\shortstack{XLSum (es) \\ (LLM: 0.229) \\ 4,763 samples}}
& Random & $0.301 \pm 0.001$ & - & 90 & $0.454 \pm 0.001$ & - & 176 & $0.609 \pm 0.001$ & - & 261 & $0.763 \pm 0.001$ & - & 347 \\
& CDI & $0.303 \pm 0.001$ & - & 95 & $0.456 \pm 0.002$ & - & 190 & $0.612 \pm 0.002$ & - & 285 & $0.768 \pm 0.001$ & - & 380 \\
& CoAnnotating & $\underline{0.304 \pm 0.001}$ & - & 95 & $0.459 \pm 0.001$  & - & 190 & $0.614 \pm 0.001$ & - & 285 & $0.768 \pm 0.001$ & - & 380 \\
& HyPAC & $0.304 \pm 0.001$ & - & 95 & $\underline{0.460 \pm 0.001}$  & - & 190 & $\underline{0.614 \pm 0.001}$ & - & 285 & $\underline{0.768 \pm 0.001}$ & - & 380 \\
& \approach & $\mathbf{0.307 \pm 0.001}$ & - & \textbf{90} & $\mathbf{0.461 \pm 0.001}$  & - & \textbf{176} & $\mathbf{0.616 \pm 0.001}$ & - & \textbf{261} & $\mathbf{0.771 \pm 0.001}$ & - & \textbf{347} \\
\bottomrule
\end{tabular}}
\end{table*}

\begin{table*}[!t]
\centering
\caption{Performance in the \textsc{dollars} setting under a fixed monetary budget. We report Accuracy or ROUGE-1 depending on the task ($\uparrow$) and Machine Cumulative Accuracy ($\uparrow$) under varying budget constraints. \textbf{Bold} indicates the best-performing method, while \underline{underlined} the second best. Row values indicate the available annotation budget, expressed as a percentage of the cost required to annotate the entire dataset using human annotators.}
\label{tab:japanesedollars}
\resizebox{\linewidth}{!}{%
\begin{tabular}{llcc|cc|cc|cc}
\toprule
\multirow{2}{*}{\textbf{Dataset}} & \multirow{2}{*}{\textbf{Method}} & \multicolumn{2}{c}{\textbf{10\%}} & \multicolumn{2}{c}{\textbf{30\%}} & \multicolumn{2}{c}{\textbf{50\%}} & \multicolumn{2}{c}{\textbf{70\%}} \\
\cmidrule(lr){3-4} \cmidrule(lr){5-6} \cmidrule(lr){7-8} \cmidrule(lr){9-10}
& & Accuracy/ROUGE & MCA & Accuracy/ROUGE & MCA  & Accuracy/ROUGE & MCA & Accuracy/ROUGE & MCA \\ \midrule
\multirow{7}{*}{\shortstack{Global-MMLU (ja) \\ (LLM: 0.640) \\ Human Cost: 1,404}}
& Random & $0.665 \pm 0.002$ & $0.313 \pm 0.008$ & $0.659 \pm 0.002$ & $0.343 \pm 0.007$ & $0.729 \pm 0.002$ & $0.245 \pm 0.003$ & $0.766 \pm 0.002$ & $0.243 \pm 0.002$ \\
& ARAIDA & $0.696 \pm 0.028$ & $0.354 \pm 0.109$ & $0.732 \pm 0.018$ & $0.312 \pm 0.012$ & $0.759 \pm 0.052$ & $0.315 \pm 0.015$ & $0.782 \pm 0.082$ & $0.318 \pm 0.018$ \\
& CDI & $0.726 \pm 0.002$ & $0.350 \pm 0.008$ & $0.751 \pm 0.002$ & $0.318 \pm 0.006$ & $0.823 \pm 0.002$ & $\underline{0.335 \pm 0.004}$ & $0.893 \pm 0.002$ & $0.338 \pm 0.002$ \\
& CoAnnotating & $0.727 \pm 0.001$ & $0.353 \pm 0.001$ & $0.760 \pm 0.001$ & $0.348 \pm 0.001$ & $0.818 \pm 0.001$ & $0.323 \pm 0.001$ & $\underline{0.909 \pm 0.001}$ & $\underline{0.361 \pm 0.001}$ \\
& HyPAC & $\underline{0.749 \pm 0.001}$ & $0.351 \pm 0.004$ & $\underline{0.784 \pm 0.001}$ & $\underline{0.350 \pm 0.002}$ & $\mathbf{0.838 \pm 0.001}$ & $0.331 \pm 0.002$ & $0.899 \pm 0.010$ & $0.342 \pm 0.008$ \\
& SANT & $0.719 \pm 0.021$ & $\underline{0.381 \pm 0.075}$ & $0.738 \pm 0.019$ & $0.307 \pm 0.044$ & $0.761 \pm 0.038$ & $\mathbf{0.335 \pm 0.058}$ & $0.780 \pm 0.064$ & $0.293 \pm 0.010$ \\
& \approach & $\mathbf{0.758 \pm 0.001}$ & $\mathbf{0.393 \pm 0.003}$ & $\mathbf{0.791 \pm 0.001}$ & $\mathbf{0.405 \pm 0.005}$ & $\underline{0.808 \pm 0.001}$ & $0.331 \pm 0.003$ & $\mathbf{0.916 \pm 0.001}$ & $\mathbf{0.382 \pm 0.002}$ \\ \midrule
\multirow{7}{*}{\shortstack{Global-MMLU (es) \\ (LLM: 0.676) \\ Human Cost: 143}}
& Random & $0.612 \pm 0.001$ & $0.305 \pm 0.044$ & $0.613 \pm 0.001$ & $0.270 \pm 0.043$ & $0.635 \pm 0.001$ & $0.185 \pm 0.038$ & $0.624 \pm 0.001$ & $0.206 \pm 0.014$ \\
& ARAIDA & $0.712 \pm 0.002$ & $\mathbf{0.844 \pm 0.002}$ & $0.714 \pm 0.002$ & $\mathbf{0.848 \pm 0.002}$ & $0.716 \pm 0.002$ & $\mathbf{0.557 \pm 0.002}$ & $0.718 \pm 0.002$ & $\underline{0.362 \pm 0.002}$ \\
& CDI & $0.712 \pm 0.001$ & $\underline{0.844 \pm 0.001}$ & $0.714 \pm 0.001$ & $\underline{0.848 \pm 0.001}$ & $0.716 \pm 0.001$ & $\underline{0.557 \pm 0.001}$ & $0.718 \pm 0.001$ & $0.362 \pm 0.001$ \\
& CoAnnotating & $\mathbf{0.727 \pm 0.001}$ & $0.376 \pm 0.001$ & $\underline{0.742 \pm 0.001}$ & $0.425 \pm 0.001$ & $\underline{0.771 \pm 0.001}$ & $0.450 \pm 0.001$ & $\underline{0.818 \pm 0.001}$ & $0.320 \pm 0.001$ \\
& HyPAC & $0.717 \pm 0.001$ & $0.286 \pm 0.025$ & $0.719 \pm 0.001$ & $0.298 \pm 0.021$ & $0.724 \pm 0.001$ & $0.301 \pm 0.019$ & $0.732 \pm 0.001$ & $0.303 \pm 0.009$ \\
& SANT & $0.712 \pm 0.001$ & $0.844 \pm 0.001$ & $0.714 \pm 0.001$ & $0.848 \pm 0.001$ & $0.716 \pm 0.001$ & $0.557 \pm 0.001$ & $0.718 \pm 0.001$ & $0.362 \pm 0.001$ \\
& \textbf{\approach} & $\underline{0.725 \pm 0.001}$ & $0.749 \pm 0.001$ & $\mathbf{0.806 \pm 0.001}$ & $0.745 \pm 0.001$ & $\mathbf{0.808 \pm 0.001}$ & $0.471 \pm 0.001$ & $\mathbf{0.829 \pm 0.001}$ & $\mathbf{0.386 \pm 0.002}$ \\ \midrule
& & ROUGE-1 & MCA & ROUGE-1 & MCA  & ROUGE-1 & MCA & ROUGE-1 & MCA \\ \midrule
\multirow{5}{*}{\shortstack{XLSum (ja) \\ (LLM: 0.166) \\ Human Cost: 88}}
& Random & $0.234 \pm 0.004$ & - & $0.298 \pm 0.006$ & -  & $0.432 \pm 0.005$ & -  & $0.508 \pm 0.001$ & -  \\
& CDI & $\underline{0.323 \pm 0.003}$ & -  & $0.381 \pm 0.004$ & -  & $\underline{0.529 \pm 0.004}$ & -  & $0.673 \pm 0.003$ & -  \\
& CoAnnotating & $0.326 \pm 0.001$ & -  & $\underline{0.393 \pm 0.001}$ & - & $0.521 \pm 0.001$ & -  & $\underline{0.673 \pm 0.001}$ & -  \\
& HyPAC & $0.319 \pm 0.004$ & -  & $0.367 \pm 0.003$ & -  & $0.385 \pm 0.001$ & -  & $0.405 \pm 0.003$ & -  \\
& \approach   & $\mathbf{0.330 \pm 0.005}$ & -  & $\mathbf{0.395 \pm 0.001}$ & -  & $\mathbf{0.543 \pm 0.002}$ & -  & $\mathbf{0.683 \pm 0.002}$ & - \\ \midrule
\multirow{5}{*}{\shortstack{XLSum (es) \\ (LLM: 0.229) \\ Human Cost: 476}}
& Random & $0.219 \pm 0.001$ & - & $0.222 \pm 0.001$ & - & $0.231 \pm 0.001$ & - & $0.246 \pm 0.002$ \\
& CDI & $0.228 \pm 0.002$ & - & $0.231 \pm 0.002$ & - & $0.236 \pm 0.002$ & - & $0.246 \pm 0.002$ & - \\
& CoAnnotating & $0.226 \pm 0.002$ & - & $0.228 \pm 0.002$ & - & $0.234 \pm 0.002$ & - & $0.251 \pm 0.002$ & - \\
& HyPAC & $\underline{0.243 \pm 0.002}$ & - & $\underline{0.249 \pm 0.002}$ & - & $\underline{0.261 \pm 0.002}$ & - & $\underline{0.284 \pm 0.001}$ & -\\
& \approach & $\mathbf{0.251 \pm 0.001}$ & - & $\mathbf{0.261 \pm 0.001}$ & - & $\mathbf{0.282 \pm 0.001}$ & - & $\mathbf{0.308 \pm 0.001}$ & -\\
\bottomrule
\end{tabular}}
\end{table*}

\end{document}